\documentclass[11pt,a4paper]{article}

\usepackage[margin=1.2in]{geometry}

\usepackage[utf8]{inputenc}
\usepackage{amsmath}
\usepackage{mathtools}
\usepackage{amsfonts}
\usepackage{amssymb}
\usepackage[]{algorithm}
\usepackage{algorithmic}
\usepackage{amsthm}
\usepackage{hyperref}
\usepackage[noabbrev,capitalize]{cleveref}
\usepackage[english]{babel}
\usepackage{comment}
\usepackage{array,multirow}
\usepackage{graphicx}
\graphicspath{ {./images/} }
\usepackage{subcaption}
\usepackage[section]{placeins}
\usepackage{color}

\newtheorem{theorem} {Theorem}
\newtheorem{lemma} {Lemma}
\newtheorem{definition} {Definition}
\newtheorem{corollary} {Corollary}

\newtheorem{assumption} {Assumption}

\def\v{{\mathbf{v}}}

\def\y{{\mathbf{y}}}

\def\X{{\mathbf{X}}}
\def\Y{{\mathbf{Y}}}

\def\A{{\mathbf{A}}}
\def\M{{\mathbf{M}}}
\def\N{{\mathbf{N}}}
\def\I{{\mathbf{I}}}

\def\V{{\mathbf{V}}}

\def\Q{{\mathbf{Q}}}

\DeclareMathOperator*{\argmin}{arg\,min}

\newcommand{\mA}{\mathcal{A}}

\newcommand{\mbS}{\mathbb{S}}

\newcommand{\mD}{\mathcal{D}}
\newcommand{\E}{\mathbb{E}}

\newcommand{\nnz}{\textrm{nnz}}
\newcommand{\trace}{\textrm{Tr}}
\newcommand{\rank}{\textrm{rank}}

\newcommand{\reals}{\mathbb{R}}

\title{Fast Stochastic Algorithms for Low-rank and Nonsmooth Matrix Problems}
\author{Dan Garber \\ {\small Technion - Israel Institute of Technology}\\ {\small \texttt{dangar@technion.ac.il}}
\and
Atara Kaplan \\  {\small Technion - Israel Institute of Technology} \\ {\small  \texttt{ataragold@technion.ac.il}}
}

\date{}

\begin{document} 

\maketitle

\begin{abstract}
Composite convex optimization problems which include both a nonsmooth term and a low-rank promoting term have important applications in machine learning and signal processing, such as when one wishes to recover an unknown matrix that is  simultaneously low-rank and sparse. However, such problems are highly challenging to solve in large-scale: the low-rank promoting term prohibits efficient implementations of proximal methods for composite optimization and even simple subgradient methods. On the other hand, methods which are tailored for low-rank optimization, such as conditional gradient-type methods, which are often applied to a smooth approximation of the nonsmooth objective, are slow since their runtime scales with both the large Lipshitz parameter of the smoothed gradient vector and with $1/\epsilon$.


In this paper we develop efficient algorithms for \textit{stochastic} optimization of a strongly-convex objective which includes both a nonsmooth term and a low-rank promoting term. In particular, to the best of our knowledge, we present the first algorithm that enjoys all following critical properties for large-scale problems: i) (nearly) optimal sample complexity, ii) each iteration requires only a single \textit{low-rank} SVD computation, and iii) overall number of thin-SVD computations scales only with $\log{1/\epsilon}$ (as opposed to $\textrm{poly}(1/\epsilon)$ in previous methods). We also give an algorithm for the closely-related finite-sum setting. At the heart  of our results lie a novel combination of a variance-reduction technique and the use of a \textit{weak-proximal oracle} which is key to obtaining all above three properties simultaneously.
We empirically demonstrate our results on the problem of recovering a simultaneously low-rank and sparse matrix.
Finally, while our main motivation comes from low-rank matrix optimization problems, our results apply in a much wider setting, namely when a \textit{weak proximal oracle} can be implemented much more efficiently than the standard exact proximal oracle.


\end{abstract}

\section{Introduction}

Our paper is strongly motivated by low-rank and non-smooth matrix optimization problems which are quite common in machine learning and signal processing applications. These include tasks such as \textit{low-rank and sparse covariance matrix estimation}, \textit{graph denoising} and \textit{link prediction} \cite{Richard12}, \textit{analysis of social networks} \cite{Zhou13}, and \textit{subspace clustering} \cite{Wang13}, to name a few.


Such optimization problems often fit the following very general optimization model:
\begin{equation} \label{eq:nonSmoothModel}
\min_{\X\in \mathbb{V}}{f(\X):=G(\X)+R^{\textrm{NS}}(\X)+h(\X)},
\end{equation}

where $\mathbb{V}$ is a finite linear space over the reals, $G(\cdot)$ is convex and smooth, $R^{\textrm{NS}}(\cdot)$ is convex and (generally) nonsmooth, and $h(\cdot)$ is convex and proximal-friendly (e.g., it is an indicator function for a convex set or a convex regularizer). Motivated by large-scale machine learning settings, we further assume $G(\cdot)$ is stochastic, i.e., $G(\X) = \E_{g\sim\mD}[g(\X)]$, where $\mD$ is a distribution over convex and smooth functions, and either given by a sampling oracle (stochastic setting), or admits a finite support and given explicitly (finite-sum setting). Finally, we assume $f(\cdot)$ is strongly-convex (either due to strong convexity of $G(\cdot)$ or $R^{\textrm{NS}}(\cdot)$). For instance the simultaneously low-rank and sparse covariance estimation problem \cite{Richard12} can be written as:

\begin{equation} \label{eq:matrixEstimation_intro}
\min_{tr(\X)\le\tau, \ \X\succeq 0}{\frac{1}{2}\Vert \X-\M\Vert_{F}^{2}+\lambda\Vert \X\Vert_{1}},
\end{equation}
where $\M = \Y\Y^{\top} + \N$ is a noisy observation of some low-rank and sparse covariance matrix $\Y\Y^{\top}$. Here, $\mathbb{V} = \mbS_n$ (space of $n\times n$ real symmetric matrices), $G(\X) =\frac{1}{2}\Vert{\X-\M}\Vert_F^2$ (which is deterministic in this simple example), $R^{\textrm{NS}}(\X) = \lambda\Vert{\X}\Vert_1$, and $h(\X)$ is an indicator function for the trace-bounded positive semidefinite cone (which both constraints the solution to be positive semidefinite and promotes low-rank). A closely related problem to \eqref{eq:matrixEstimation_intro} to which all of the following discussions apply, is when $\X\in\mathbb{V}=\reals^{m\times n}$ is not constrained to be positive semidefinite (or even symmetric), and a low-rank solution is encouraged by constraining $\X$ via a nuclear norm constraint $\Vert{\X}\Vert_*\leq \tau$, where $\Vert\cdot\Vert_*$ is the $\ell_1$ norm applied to the vector of singular values.

In Table \ref{table:l1_importance} we provide a very simple numerical demonstration of the applicability of Problem \eqref{eq:matrixEstimation_intro} to low-rank and sparse estimation, which exhibits the importance of combining both low-rank and entry-wise sparsity promoting terms (as opposed to methods that only promote low-rank).

\begin{table*}[!htb]\renewcommand{\arraystretch}{1.3}
{\small
\begin{center}
  \begin{tabular}{| c | l | c | c | c |} \hline 
    Noise & & Low Rank & Projection & Low Rank \\
    Level ($c$) & &(1-SVD) & (eq. (\ref{eq:matrixEstimation_intro}) w. $\lambda =0$) & \& Sparse (eq. (\ref{eq:matrixEstimation_intro})) \\ \hline
    \multirow{3}{*}{0.5}
    & $\Vert \X^*-\y\y^{\top}\Vert_F^2/\Vert \y\y^{\top}\Vert_F^2$ & $0.0009$ & $0.0027$ & $\mathbf{0.0003}$  \\ \cline{2-5}
    & $\nnz(\X^*)/\nnz(\y\y^{\top})$ & $105.625$ & $105.625$ & $\mathbf{1}$  \\ \cline{2-5}
    & $\rank(\X^*)$ & $1$& $2$ & $2.5$   \\ \hline
    \multirow{3}{*}{5}
    & $\Vert \X^*-\y\y^{\top}\Vert_F^2/\Vert \y\y^{\top}\Vert_F^2$ & $1.0991$  & $0.3794$ & $\mathbf{0.01}$ \\ \cline{2-5}
    & $\nnz(\X^*)/\nnz(\y\y^{\top})$ & $114.8056$ & $114.8056$ & $\mathbf{1.0017}$  \\ \cline{2-5}
    & $\rank(\X^*)$ & $1$&  $1.96$ & $2.54$  \\ \hline
  \end{tabular}
  \caption{Numerical example (showing the signal-recovery error, relative sparsity and rank of solution $\X^*$) for estimating a sparse rank-one matrix $\Y\Y^{\top}=\y\y^{\top}$ from the noisy observation $\M = \y\y^{\top}+ \frac{c}{2}(\N+\N^{\top})$, where $\N\sim\mathcal{N}(0,\I_n)$. Each entry $\y_i$ is zero w.p. $0.9$ and $U\{1,\dots,10\}$ w.p. $0.1$. The dimension is $n=50$.  Results are averages of 50 i.i.d. experiments, and $c\in\{0.5,5\}$ (magnitude of noise). For method \eqref{eq:matrixEstimation_intro} $\lambda$ is chosen via experimentation. 
  }\label{table:l1_importance}
\end{center}
}
\vskip -0.2in
\end{table*}\renewcommand{\arraystretch}{1} 

The general model \eqref{eq:nonSmoothModel} is known to be a very difficult optimization problem to solve in large scale, already in the specific setting of Problem \eqref{eq:matrixEstimation_intro}. 
In particular, many of the traditional first-order convex optimization methods used for solving non-smooth optimization problems are not efficiently applicable to it. For instance, proximal methods for composite optimization, such as the celebrated FISTA algorithm \cite{FISTA}, which in many cases are very efficient, do not admit efficient implementations for composite problems which include both a non-smooth term and a low-rank promoting term. When applied to Problem \eqref{eq:matrixEstimation_intro}, each iteration of FISTA will require to solve a problem of the same form as the original problem, and hence is inefficient. 
Another type of well known first-order methods that are applicable to nonsmooth problems are deterministic/stochastic subgradient/mirror-descent methods \cite{Nesterov13, Bubeck15}. However these methods are also inefficient for problems such as \eqref{eq:matrixEstimation_intro}, since each iteration requires projecting a point onto the feasible set, which for problems such as \eqref{eq:matrixEstimation_intro}, requires a full-rank SVD computation on each iteration, which is computationally-prohibitive for large-scale problems.

Another type of methods, which are often suitable for large-scale low-rank matrix optimization problems, and have been studied extensively in this context in recent years, are Conditional Gradient-type methods (aka Frank Wolfe-type methods), see for instance \cite{Jaggi13,Garber16,k-SVD,Hazan_finiteSum,Garber18,Cevher18,Simon18,Simon17,mu2016scalable}. These type of algorithms, when applied to optimization over a nuclear-norm ball or over the trace-bounded positive semidefinite cone (as in Problem \eqref{eq:matrixEstimation_intro}), avoid expensive full-rank SVD computations, and only compute a single leading singular vector pair on each iteration (i.e., rank-one SVD), and hence are much more scalable. However,  Conditional Gradient methods can usually be applied only to smooth problems, and so, the non-smooth term $R^{\textrm{NS}}(\X)$ is often replaced with a smooth approximation $R(\X)$. A general theory and framework for generating such smooth approximation (i.e., replacing the non-smooth term with a smooth function that is point-wise close to the original), often referred to as \textit{smoothing}, is described in \cite{smoothing}. Unfortunately, smoothing a function often results in a large Lipschitz constant of the gradient vector of the smoothed function. For example, the smooth approximation of the $\ell_1$ norm is via the well known Huber function for which the Lipschitz constant of the gradient often scales like $\dim(\mathbb{V})/\varepsilon$, where $\varepsilon$ is target accuracy to which the problem needs to be solved. Since the convergence rate of smooth optimization algorithms such as conditional gradient-type methods discussed above often scales with $\beta{}D^2/\epsilon$, where $\beta$ is
the Lipschitz parameter of the gradient and $D$ is the distance of the initial point to an optimal solution,
these methods are often not scalable for nonsmooth objectives such as Problem \eqref{eq:matrixEstimation_intro} and the general model \eqref{eq:nonSmoothModel} (even after smoothing them), since typically all three parameters $1/\varepsilon,D,\beta$ can be quite large. In particular, we note that for strongly-convex functions, it is possible to obtain (via other types of first-order methods) rates that depend only logarithmically on $1/\epsilon, D$.

Another issue with conditional gradient methods is that, as opposed to projected subgradient methods, their analysis does not naturally extend to handle stochastic objectives (recall that, motivated by machine learning settings, in the general model \eqref{eq:nonSmoothModel} we assume $G(\cdot)$ is stochastic). In particular, a straightforward variant of the method for stochastic objectives results in a highly suboptimal sample complexity \cite{Hazan_finiteSum}. In a recent related work \cite{Lan}, the authors consider a variant of the conditional gradient method for solving stochastic optimization problems that cleverly combines the conditional gradient method with Nesterov's accelerated method and stochastic sampling to obtain an algorithm for smooth stochastic convex optimization that, in the context of low-rank matrix optimization problems, i) requires only 1-SVD computation on each iteration (as in the standard conditional gradient method) and ii) enjoys (nearly) optimal sample complexity (both in the strongly convex case and non-strongly convex case). In a recent work \cite{Hazan_finiteSum}, the technique of \cite{Lan} was extended to the finite-sum stochastic setting and combined with a popular variance reduction technique \cite{SVRG}, resulting in a conditional gradient-type method for smooth and strongly-convex finite-sum optimization that i) requires only 1-SVD computation on each iteration, and ii) enjoys a gradient-oracle complexity of the same flavor as usually obtained via variance-reduction methods \cite{SVRG}, greatly improving over naive applications of conditional gradient methods which do not apply variance reduction. Unfortunately, both results \cite{Lan,Hazan_finiteSum}, while greatly improving the first-order oracle complexity of previous conditional-gradient methods, still require an overall number of 1-SVD computations that scales like $\beta{}D/\epsilon$. Hence, when applied to smooth approximations of nonsmooth problems such as Problems \eqref{eq:matrixEstimation_intro}, \eqref{eq:nonSmoothModel}, the overall very large number of thin-SVD computations needed greatly limits the applicability of these methods.

The limitations of previous methods in tackling large-scale low-rank and nonsmooth matrix optimization problems naturally leads us to the following question. \\ \\
\noindent\textit{
In the context of low-rank and nonsmooth matrix optimization, is it possible to combine all following three key properties for solving large-scale instances of Model \eqref{eq:nonSmoothModel} into a single algorithm?
\begin{enumerate}
\item (nearly) optimal sample complexity,
\item use of only low-rank SVD computations,
\item overall number of low-rank SVD computations scales with $\log(1/\varepsilon)$ (not $\textrm{poly}(1/\epsilon)$ as in previous methods).
\end{enumerate}
}
In this paper we answer this question in the affirmative. To better discuss our results we now fully formalize the considered model and assumptions.

We consider the following general model:
\begin{equation} \label{eq:generalModel}
\min_{\X\in \mathbb{V}}{f(\X):=G(\X)+R(\X)+h(\X)},
\end{equation}
where $\mathbb{V}$ is a finite linear space over the reals equipped with an inner product $\langle\cdot,\cdot\rangle$. Throughout the paper we let $\Vert\cdot\Vert$ denote the norm induced by the inner product.
%

Throughout the paper we consider the following assumptions for model (\ref{eq:generalModel}).
\begin{assumption} \label{Ass1}
\begin{itemize} 
\item $G$ is stochastic, i.e., $G(\X)=\mathbb{E}_{g\sim \mathcal{D}}[g(\X)]$, where $\mathcal{D}$ is a distribution over functions $g:\mathbb{V}\rightarrow\mathbb{R}$, given by a sampling oracle. $G$ is differentiable, and for all $g\in supp(\mathcal{D})$, $g$ is $\beta_{G}$-smooth, and there exists $\sigma\ge0$ such that $\sigma \ge \sup\limits_{\X\in dom(h)}\sqrt{\mathbb{E}[\Vert\nabla G(\X)-\nabla g(\X)\Vert^{2}]}$.
\item $R:\mathbb{V}\rightarrow (-\infty,\infty]$ is deterministic, $\beta_{R}$-smooth, and convex,
\item $G+R$ is $\alpha$-strongly convex,
\item $h:\mathbb{V}\rightarrow (-\infty,\infty]$ is deterministic, non-smooth, proper, lower semicontinuous and convex.
\end{itemize}
\end{assumption} 

For simplicity we define $\beta:=\beta_G+\beta_R$. As discussed above, $R(\cdot)$ can be thought of as a smooth approximation of some nonsmooth term $R^{\textrm{NS}}(\cdot)$ (hence, we generally expect that $\beta_R >> \beta_G$), and $h(\cdot)$ can be thought of as either an indicator function for a convex set (e.g., a nuclear-norm ball) or a convex regularizer.

A quick summary of our results and comparison to previous conditional gradient-type methods for solving Model \eqref{eq:generalModel} in case $h(\cdot)$ is either an indicator for a nuclear norm ball of radius $\tau$ or the set of all positive semidefinite matrices with trace at most $\tau$,  is given in Table \ref{table:algorithms}.

Our algorithm and novel complexity bounds are based on a combination of  the variance reduction technique introduced in \cite{SVRG} and the use of, what we refer to in this work as, a \textit{weak-proximal oracle} (as opposed to the standard exact proximal oracle used ubiquitously in first-order methods), which was introduced in the context of nuclear-norm-constrained optimization in \cite{k-SVD}, and further generalized in \cite{Garber18}. In the context of low-rank matrix optimization problems, implementation of this weak-proximal oracle requires a SVD computation of rank at most $\rank(\X^*)$ - the rank of the optimal solution $\X^*$, as opposed to an exact proximal oracle that requires in general a full-rank SVD computation. Since for such problems we expect that $\rank(\X^*)$ is much smaller than the dimension, and since the runtime of low-rank SVD computations (when carried out via fast iterative methods such as variants of the subspace iteration method or Lanczos-type algorithms) scales nicely with both the target rank and sparsity of gradients\footnote{see for instance discussions in \cite{k-SVD}.}, for such problems the weak-proximal oracle admits a much more efficient implementation than the standard proximal oracle. 

While both of these algorithmic ingredients are previously known and studied, it is their particular combination that, quite surprisingly, proves to be key to obtaining all three complexity bounds listed in our proposed question, simultaneously. In particular, it is important to note that while the use of a weak proximal oracle, as we define precisely in the sequel, suffices to obtain an algorithm that uses overall only $O(\log(1/\epsilon$) \textit{low-rank} SVD computations (currently treating for simplicity all other parameters as constants), to the best of our knowledge it does not suffice in order to also obtain (nearly) optimal sample complexity. The reason, at a high-level (see a more detailed discussion in the sequel), is that the weak-proximal oracle is strong enough to guarantee decrease of the loss function on each iteration (in expectation), but does not give a stronger type of guarantee, which holds for the exact proximal oracle, that is crucial for obtaining optimal sample complexity with algorithms such as Stochastic Gradient Descent \cite{Bubeck15} and the conditional gradient-type method of \cite{Lan} (that indeed rely on exact, or nearly exact, proximal computations). It turns out that the use of a variance reduction technique (such as  \cite{SVRG}) is key to bypassing this obstacle and obtaining also (near) optimal sample complexity, on top of the low SVD complexity. We also give a variant of our algorithm to the finite-sum setting that obtains similar improvements.

\begin{table*}[!htb]\renewcommand{\arraystretch}{1.4}
{\small
\begin{center}
  \begin{tabular}{| l | c | c | c | c |} \hline
    Algorithm  & $\#$Exact & $\#$Stochastic & SVD& $\#$SVD \\
    &  Gradients &  Gradients &  rank & Computations  \\ \hline 
     \multicolumn{5}{|c|}{$\downarrow$ Stochastic Setting $\downarrow$}\\ \hline
      Stochastic Cond. Grad. \cite{Hazan_finiteSum}  & $0$ & $\frac{\sigma^2\beta \tau^4}{\varepsilon^3}$ & $1$& $\frac{\beta \tau^{2}}{\varepsilon}$  \\ \hline
     CGS \cite{Lan}  & $0$ & $\frac{\sigma^2}{\alpha\varepsilon}$ &  $1$ & $\frac{\beta \tau^{2}}{\varepsilon}$  \\ \hline
    This work (Alg. \ref{alg:stoch})  & $0$ & $\frac{\sigma^2}{\alpha\varepsilon}$ & $\rank(\X^*)$ &$\frac{\beta}{\alpha}\ln{\left(\frac{1}{\varepsilon}\right)}$ \\ \hline
         \multicolumn{5}{|c|}{$\downarrow$ Finite Sum $\downarrow$}\\ \hline
    STORC \cite{Hazan_finiteSum}  & $\ln{\left(\frac{1}{\varepsilon}\right)}$ & $\left(\frac{\beta}{\alpha}\right)^{2}\ln{\left(\frac{1}{\varepsilon}\right)}$ & $1$& $\frac{\beta \tau^2}{\varepsilon}$  \\ \hline
    This work finite sum (Alg. \ref{alg:stochFiniteSum}) & $\ln{\left(\frac{1}{\varepsilon}\right)}$ & $\frac{\beta_G^2\beta}{\alpha^3}\ln{\left(\frac{1}{\varepsilon}\right)}$ &  $\rank(\X^*)$  &$\frac{\beta}{\alpha}\ln{\left(\frac{1}{\varepsilon}\right)}$ \\ \hline
  \end{tabular}
  \caption{Comparison of complexity bound for conditional gradient-type methods for solving Model \eqref{eq:generalModel}. $\X^*$ denotes the unique optimal solution. Table only lists the leading-order terms.
  }\label{table:algorithms}
\end{center}
}
\vskip -0.2in
\end{table*}\renewcommand{\arraystretch}{1}

Finally, while our main motivation comes from low-rank and nomsmooth matrix optimization problems, it is important to note that as captured in our general Model \eqref{eq:generalModel}, our results are applicable in a much wider setting than that of low-rank matrix optimization problems. Our method is suitable especially for stochastic nonsmooth convex problems for which implementing a weak proximal oracle is much more efficient than an exact proximal oracle.

\subsection{Organization of this paper}

The rest of this paper is organized as follows. In Section \ref{sec:algNresults} we present our main algorithm and our two main results (informally). Importantly, we discuss in detail the importance of combining stochastic variance reduction with a weak proximal oracle to obtain our novel complexity bounds. In section \ref{sec:analysis} we describe our main results in full detail and prove them. In Section \ref{sec:smoothing} we describe in detail applications of our results to non-smooth optimization problems including several concrete examples. Finally, in Section \ref{sec:experiments} we present preliminary empirical evidence supporting our theoretical results.

\section{Algorithm and Results}\label{sec:algNresults}

Our algorithm for solving Model \eqref{eq:generalModel}, Algorithm \ref{alg:stoch}, is given below. We now briefly discuss the main two building blocks of the algorithm, namely a variance reduction technique and the use of a weak proximal oracle, and the importance of their combination in achieving the novel complexity bounds.

\subsection{The importance of combining weak proximal updates with variance reduction}
Our use of the variance reduction technique of \cite{SVRG} is quite straightforward as observable in Algorithm \ref{alg:stoch}. Importantly, while \cite{SVRG} applied it to finite-sum optimization, here we apply it to the more general  black-box stochastic setting, and hence the sample-size parameter $k_s$ used for the "snap-shot" gradient $\tilde{\nabla}g(\X_s)$ on epoch $s$ grows from epoch to epoch. This modification of the technique is along the lines of \cite{Frostig15}.

The weak proximal oracle strategy is applied in our algorithm as follows. For a step-size $\eta_t$, a composite optimization proximal algorithm, which treats the function $h(\cdot)$ in proximal fashion and the functions $G,R$ via a gradient oracle, will compute on each iteration an update of the form
\begin{eqnarray}\label{eq:prox}
\V_t \gets\arg\min_{\V\in\mathbb{V}}\left\{ \psi_t(\V):= \Vert \V-\X_{s,t}+\frac{1}{2\beta\eta_t}(\hat{\nabla}g(\X_{s,t})+\nabla R(\X_{s,t}))\Vert^{2}+\frac{1}{\beta\eta_t}h(\V)\right\}.
\end{eqnarray}
For instance, if $\mathbb{V} = \reals^{m\times n}$ and $h(\cdot)$ is an indicator function for the nuclear-norm ball $\{\X\in\reals^{m\times n} ~ |~\Vert{\X}\Vert_*\leq \tau\}$, then computing $\V_t$ in Eq. \eqref{eq:prox} amounts to Euclidean projection of the matrix  $\A_t = \X_{s,t}-\frac{1}{2\beta\eta_t}(\hat{\nabla}g(\X_{s,t})-\nabla R(\X_{s,t}))$ onto the nuclear-norm ball of radius $\tau$. This projection is carried out by computing a full-rank SVD of $\A_t$ and projecting the singular values onto the $\tau$-scaled simplex. Since a full-rank SVD is required, this operation takes $O(m^2n)$ time (assuming $m\leq n$), which is prohibitive for very large $m,n$.

Our algorithm avoids the computational bottleneck of full-rank SVD computations by only requiring that $\V_t$ satisfies the inequality:
\begin{eqnarray}\label{eq:weakprox}
\psi_t(\V_t) \leq \psi_t(\X^*),
\end{eqnarray}
where $\X^*$ is the (unique) optimal solution to \eqref{eq:generalModel}. We call a procedure for computing such updates -  a weak proximal oracle. 
In the context discussed above, i.e., $h(\cdot)$ is an indicator for the radius-$\tau$ nuclear-norm ball, \eqref{eq:weakprox} can be satisfied simply by projecting the $\rank(\X^*)$-approximation of the matrix $\A_t$ onto the nuclear-norm ball. This only requires to compute the top $\rank(\X^*)$ components in the singular value decomposition of $\A_t$, and thus the runtime scales roughly like $O(\rank(\X^*)\cdot\nnz(\A_t))$ using fast Krylov Subspace methods (e.g., subspace iteration, Lanczos), which results in a much more efficient procedure (see further detailed discussions in \cite{k-SVD,Garber18}).

Unfortunately, the use of weak proximal updates given by Eq. \eqref{eq:weakprox}, as opposed to the standard update in Eq. \eqref{eq:prox}, seems to come with a price.
While the weak-proximal guarantee $\psi_t(\V_t) \leq \psi_t(\X^*)$ is sufficient to retain the convergence rates attainable via descent-type methods, i.e., methods that decrease the function value on each iteration (see for instance \cite{k-SVD,Garber18}), it does not seem strong enough to obtain the rates of non-descent-type methods such as Nesterov's acceleration-based methods \cite{Lan}, and stochastic (sub)gradient methods \cite{Bubeck15}. The analyses of these methods seem to crucially depend on the stronger inequality
\begin{eqnarray}\label{eq:strongIneq}
\Vert{\X^*-\V_t}\Vert^2 \leq \frac{2}{\alpha_t}\left({\psi_t(\X^*)-\psi_t(\V_t)}\right),
\end{eqnarray}
where $\alpha_t$ is the strong-convexity parameter associated with the function $\psi_t(\cdot)$. The inequality \eqref{eq:strongIneq} is obtained only for an optimal minimizer of $\psi_t(\cdot)$ (as given by \eqref{eq:prox}) and not by the weak proximal solution given by \eqref{eq:weakprox}.
It is for this reason that simply combining the use of the weak proximal update \eqref{eq:weakprox} with standard analysis of SGD \cite{Bubeck15} or the Stochastic Conditional Gradient Sliding method \cite{Lan} will not result in optimal sample complexity\footnote{in particular, this suboptimal sample complexity will scale both with $\beta$ - the overall gradient Lipschitz parameter and with $1/\epsilon$, whereas the optimal sample complexity is independent of $\beta$ (which as we recall, is typically quite large in our setting due to $R(\cdot)$).}. 

Perhaps surprisingly, as our analysis shows, it is the combination of the weak proximal updates with the variance reduction technique that allows us to avoid the use of the strong inequality \eqref{eq:strongIneq} and to obtain (nearly) optimal sample complexity using only the weak proximal update guarantee \eqref{eq:weakprox}.

Since in many settings of interest, especially in the context of matrix optimization problems, the computation of $\V_t$ requires some numeric procedure which is prone to accuracy issues, or in cases in which $\X^*$ is not low-rank but only very close to a low-rank matrix (in some norm), we introduce an error-tolerance parameter $\delta$ in the proximal computation step in Algorithm \ref{alg:stoch} which allows to absorb such errors that can be controlled (e.g., by properly tuning precision of the thin-SVD computation).

\begin{algorithm}[H]
	\caption{Stochastic Variance-Reduced Generalized Conditional Gradient for Problem \eqref{eq:generalModel}}
	\label{alg:stoch}
	\begin{algorithmic}
		\STATE  \textbf{Input:} $T$, $\{\eta_t\}_{t=1}^{T-1}\subset[0,1]$, $\{k_t\}_{t=1}^{T-1}, \{k_s\}_{s\geq 1}\subset\mathbb{N}$, $\delta\ge0$.	
		\STATE \textbf{Initialization:} Choose some $\X_{1}\in dom(h)$ . 
		\FOR {$s=1,2,...$}
		\STATE Sample $g^{(1)},...,g^{(k_s)}$ from $\mathcal{D}$.
		\STATE Define $\tilde{\nabla}g(\X_s)=\frac{1}{k_s}\sum_{i=1}^{k_s}\nabla g^{(i)}(\X_s)$ \COMMENT{snap-shot gradient}.
		\STATE $\X_{s,1}=\X_{s}$
		\FOR{$t = 1,2,...,T-1$} 
			\STATE Sample $g^{(1)},...,g^{(k_t)}$ from $\mathcal{D}$.
			\STATE Define $\hat{\nabla}g(\X_{s,t})=\frac{1}{k_t}\sum_{i=1}^{k_t}\left(\nabla g^{(i)}(\X_{s,t})-\left(\nabla g^{(i)}(\X_{s})-\tilde{\nabla}g(\X_s)\right)\right)$.
			\STATE $\V_{t}=\argmin\limits_{\V \in \mathbb{V}} \left\{ \psi_t(\V):= \Vert \V-\X_{s,t}+\frac{1}{2\beta\eta_t}(\hat{\nabla}g(\X_{s,t})+\nabla R(\X_{s,t}))\Vert^{2}+\frac{1}{\beta\eta_t}h(\V)\right\}$ 
			\\ \COMMENT{in fact it suffices that $\psi_t(\V_t) \leq \psi_t(\X^*)+\delta$ for some optimal solution $\X^*$}.
            \STATE $ \X_{s,t+1}=(1-\eta_t)\X_{s,t}+\eta_t \V_t$
         \ENDFOR
         \STATE $\X_{s+1}=\X_{s,T}$
         \ENDFOR
	\end{algorithmic}
\end{algorithm}

\subsection{Outline of main results}

We now present a concise version of our main results, Theorems \ref{thm:stocohasticRates}, \ref{thm:finiteSumRates}. In section \ref{sec:analysis} we provide the complete analysis with all the details and proofs of these theorems. Subsequent results and concrete applications to non-smooth problems follow in Section \ref{sec:smoothing}.


\begin{theorem}[stochastic setting] \label{thm:stocohasticRates}
Assume that \cref{Ass1} holds. There is an explicit choice for the parameters in Algorithm \ref{alg:stoch} for which the total number of epochs (iterations of the outer-loop)  required in order to find an $\varepsilon$-approximated solution in expectation for Problem \eqref{eq:generalModel} is bounded by
\begin{equation*}
O\left(\ln{\left(\frac{1}{\varepsilon}\right)}\right),
\end{equation*}
the total number of calls to the weak proximal oracle is bounded by 
\begin{equation*}
O\left(\frac{\beta}{\alpha}\ln\left(\frac{1}{\varepsilon}\right)\right),
\end{equation*}
and the total number of stochastic gradients sampled is bounded by
\begin{equation*}
O\left(\frac{\sigma^2}{\alpha\varepsilon}+\frac{\beta_G^2\beta}{\alpha^3}\ln{\left(\frac{1}{\varepsilon}\right)}\right).
\end{equation*}
\end{theorem}

We note that under Assumption \ref{Ass1}, the overall number of calls to a weak proximal oracle to reach $\epsilon$-approximated solution matches the overall number of calls to an \textit{exact} proximal oracle used by the proximal gradient method for smooth and strongly convex optimization. Also, under Assumption \ref{Ass1}, the leading term in the bound on overall number of stochastic gradients is optimal (up to constants).

We also present results for the related finite sum problem. Our algorithm for finite sum is very similar to \cref{alg:stoch} and is brought in section \ref{sec:finiteSum}. 

\begin{theorem}[finite-sum setting] \label{thm:finiteSumRates}
Assume that \cref{Ass1} holds and that $\mD$ is an explicitly given uniform distribution over $n$ functions. There exist an explicit choice for the parameters in Algorithm \ref{alg:stochFiniteSum} for which the total number of epochs required in order to find an $\varepsilon$-approximated solution in expectation for Problem \eqref{eq:generalModel} is bounded by
\begin{equation*}
O\left(\ln{\left(\frac{1}{\varepsilon}\right)}\right),
\end{equation*}
the total number of calls to the weak proximal oracle is bounded by 
\begin{equation*}
O\left(\frac{\beta}{\alpha}\ln{\left(\frac{1}{\varepsilon}\right)}\right),
\end{equation*}
and the total number of gradients computed for any of the $n$ functions in the support of $\mD$ is bounded by:
\begin{equation*}
O\left(\left(n+\frac{\beta_G^2\beta}{\alpha^3}\right)\ln{\left(\frac{1}{\varepsilon}\right)}\right).
\end{equation*}
\end{theorem}
We see that as is standard in variance-reduced methods for smooth and strongly convex optimization, the overall number of gradients decouples between terms that depend on the smoothness and strong convexity of the objective (e.g., the condition number $\beta/\alpha$), and the overall number of functions $n$.


\section{Analysis} \label{sec:analysis}

In this section we prove Theorem \ref{thm:stocohasticRates}, \ref{thm:finiteSumRates}.

The following lemma bounds the expected decrease in function value after a single iteration of the inner-loop in Algorithm \ref{alg:stoch}.

\begin{lemma}[expected decrease] \label{lemma:psi_error}
Assume that \cref{Ass1} holds. Fix some epoch $s$ of Algorithm \ref{alg:stoch}, and let $\{\X_{s,t}\}_{t=1}^{T+1}$, $\{\V_{t}\}_{t=1}^{T}$ be the iterates generated throughout the epoch, and suppose that $\psi_t(\V_t) \leq \psi_t(\tilde{\X})+\delta$ for some fixed feasible solution $\tilde{\X}$. Then, if $2\beta\eta_t \le \alpha$, we have that 
\begin{equation} \label{eq:t_rate}
\begin{split}
\mathbb{E}[f(\X_{s,t+1})] & \le \left(1-\eta_t\right)\mathbb{E}[f(\X_{s,t})]+\eta_t f(\tilde{\X})+\frac{\sigma_{s,t}^2}{2\beta}+\beta\eta_{t}^{2}\delta,
\end{split}
\end{equation}
where $\sigma_{s,t}=\sqrt{\mathbb{E}[\Vert \nabla G(\X_{s,t})-\hat{\nabla}g(\X_{s,t})\Vert^2]}$.
\end{lemma}

\begin{proof}
Denote $\phi:=G+R$ to be the smooth part of $f$. $\phi$ is $\beta$-smooth and so by the well known Decent Lemma,
\begin{equation*}
\begin{split}
\phi(\X_{s,t+1}) & \le \phi(\X_{s,t})+\langle{\X_{s,t+1}-\X_{s,t},\nabla \phi(\X_{s,t})}\rangle+\frac{\beta}{2}\Vert \X_{s,t+1}-\X_{s,t}\Vert^2
\\ & = \phi(\X_{s,t})+\langle{\X_{s,t+1}-\X_{s,t},\hat{\nabla} g(\X_{s,t})+\nabla R(\X_{s,t})}\rangle \\ &\ \ \ 
+\langle{\X_{s,t+1}-\X_{s,t},\nabla G(\X_{s,t})-\hat{\nabla} g(\X_{s,t})}\rangle+\frac{\beta}{2}\Vert \X_{s,t+1}-\X_{s,t}\Vert^2.
\end{split}
\end{equation*}

Plugging in $\X_{s,t+1}=(1-\eta_t)\X_{s,t}+\eta_t \V_t$, we get
\begin{equation} \label{eq:xNext}
\begin{split}
\phi(\X_{s,t+1}) & \le \phi(\X_{s,t})+\eta_t\langle{\V_{t}-\X_{s,t},\hat{\nabla} g(\X_{s,t})+\nabla R(\X_{s,t})}\rangle \\ &\ \ \ +\eta_t\langle{\V_{t}-\X_{s,t},\nabla G(\X_{s,t})-\hat{\nabla} g(\X_{s,t})}\rangle+\frac{\beta\eta_t^{2}}{2}\Vert \V_{t}-\X_{s,t}\Vert^2.
\end{split}
\end{equation}

In addition, it holds that
\begin{equation*}
\begin{split}
0 & \le \Big\|\frac{1}{\sqrt{\beta\eta_t}}(\nabla G(\X_{s,t})-\hat{\nabla} g(\X_{s,t}))-\sqrt{\beta\eta_t}(\V_t-\X_{s,t})\Big\|^{2}
\\ & = \frac{1}{\beta\eta_t}\Vert\nabla G(\X_{s,t})-\hat{\nabla} g(\X_{s,t})\Vert^{2}-2\langle{\V_t-\X_{s,t},\nabla G(\X_{s,t})-\hat{\nabla} g(\X_{s,t})}\rangle \\ &\ \ \ +\beta\eta_t\Vert \V_t-\X_{s,t}\Vert^{2}.
\end{split}
\end{equation*}

Rearranging we get,
\begin{equation*}
\langle{\V_t-\X_{s,t},\nabla G(\X_{s,t})-\hat{\nabla} g(\X_{s,t})}\rangle \le \frac{1}{2\beta\eta_t}\Vert\nabla G(\X_{s,t})-\hat{\nabla} g(\X_{s,t})\Vert^{2}+\frac{\beta\eta_t}{2}\Vert \V_t-\X_{s,t}\Vert^{2}.
\end{equation*}

Plugging this last inequality into (\ref{eq:xNext}), we get
\begin{equation*}
\begin{split}
\phi(\X_{s,t+1}) & \le \phi(\X_{s,t})+\eta_t\langle{\V_{t}-\X_{s,t},\hat{\nabla} g(\X_{s,t})+\nabla R(\X_{s,t})}\rangle \\ &\ \ \ +\eta_t\left({\frac{1}{2\beta\eta_t}\Vert\nabla G(\X_{s,t})-\hat{\nabla} g(\X_{s,t})\Vert^{2}+\frac{\beta\eta_t}{2}\Vert \V_t-\X_{s,t}\Vert^{2}}\right)+\frac{\beta\eta_t^{2}}{2}\Vert \V_{t}-\X_{s,t}\Vert^2
\\ & = \phi(\X_{s,t})+\eta_t\langle{\V_{t}-\X_{s,t},\hat{\nabla} g(\X_{s,t})+\nabla R(\X_{s,t})}\rangle \\ &\ \ \ +\frac{1}{2\beta}\Vert\nabla G(\X_{s,t})-\hat{\nabla} g(\X_{s,t})\Vert^{2}+\beta\eta_t^{2}\Vert \V_{t}-\X_{s,t}\Vert^2.
\end{split}
\end{equation*}

Using the convexity of $h$ we have that
\begin{equation*}
\begin{split}
h(\X_{s,t+1}) & = h((1-\eta_t)\X_{s,t}+\eta_t \V_t) \le (1-\eta_t)h(\X_{s,t})+\eta_t h(\V_t).
\end{split}
\end{equation*}

Combining the last two inequalities and recalling that $f=\phi+h$ we get
\begin{eqnarray*}
f(\X_{s,t+1}) & \leq & (1-\eta_t)f(\X_{s,t})+\eta_t(\phi(\X_{s,t})+h(\V_t))+\frac{1}{2\beta}\Vert\nabla G(\X_{s,t})-\hat{\nabla} g(\X_{s,t})\Vert^{2} \\ 
&&+\eta_t\langle{\V_{t}-\X_{s,t},\hat{\nabla} g(\X_{s,t})+\nabla R(\X_{s,t})}\rangle+\beta\eta_t^{2}\Vert \V_{t}-\X_{s,t}\Vert^2 \\
&=& (1-\eta_t)f(\X_{s,t}) + \eta_t\phi(\X_{s,t}) + \frac{1}{2\beta}\Vert\nabla G(\X_{s,t})-\hat{\nabla} g(\X_{s,t})\Vert^{2} \\
&&+ \beta\eta_t^2\left({\psi_t(\V_t) - \frac{1}{(2\beta\eta_t)^2}\Vert{\hat{\nabla}g(\X_{s,t}) + \nabla{}R(\X_{s,t})}\Vert^2}\right).
\end{eqnarray*}

By the definition of $\V_t$ and the assumption of the lemma we have
\begin{eqnarray*}
f(\X_{s,t+1}) & \leq & (1-\eta_t)f(\X_{s,t}) + \eta_t\phi(\X_{s,t}) + \frac{1}{2\beta}\Vert\nabla G(\X_{s,t})-\hat{\nabla} g(\X_{s,t})\Vert^{2} \\
&&+ \beta\eta_t^2\left({\psi_t(\tilde{\X}) - \frac{1}{(2\beta\eta_t)^2}\Vert{\hat{\nabla}g(\X_{s,t}) + \nabla{}R(\X_{s,t})}\Vert^2 + \delta}\right) \\
&=&  (1-\eta_t)f(\X_{s,t})+\eta_t(\phi(\X_{s,t})+h(\tilde{\X}))+\frac{1}{2\beta}\Vert\nabla G(\X_{s,t})-\hat{\nabla} g(\X_{s,t})\Vert^{2} \\ 
&&+\eta_t\langle{\tilde{\X}-\X_{s,t},\hat{\nabla} g(\X_{s,t})+\nabla R(\X_{s,t})}\rangle+\beta\eta_t^{2}\Vert \tilde{\X}-\X_{s,t}\Vert^2 + \beta\eta_{t}^{2}\delta.
\end{eqnarray*}

Taking expectation with respect to the randomness in $\hat{\nabla} g(\X_{s,t})$,
\begin{equation*}
\begin{split}
\mathbb{E}_{t}[f(\X_{s,t+1})] & \le (1-\eta_t)f(\X_{s,t})+\eta_t(\phi(\X_{s,t})+h(\tilde{\X}))+\frac{\sigma_{s,t}^2}{2\beta} \\ &\ \ \ +\eta_t\langle{\tilde{\X}-\X_{s,t},\nabla G(\X_{s,t})+\nabla R(\X_{s,t})}\rangle+\beta\eta_t^{2}\Vert \tilde{\X}-\X_{s,t}\Vert^2+\beta\eta_{t}^{2}\delta.
\end{split}
\end{equation*}

Using the $\alpha$-strong convexity of $\phi=G+R$ we get
\begin{equation*}
\begin{split}
\mathbb{E}_{t}[f(\X_{s,t+1})] & \le (1-\eta_t)f(\X_{s,t})+\eta_t(\phi(\X_{s,t})+h(\tilde{\X}))+\frac{\sigma_{s,t}^2}{2\beta} \\ &\ \ \ +\eta_t\left({\phi(\tilde{\X})-\phi(\X_{s,t})-\frac{\alpha}{2}\Vert \tilde{\X}-\X_{s,t}\Vert^2}\right)+\beta\eta_t^{2}\Vert \tilde{\X}-\X_{s,t}\Vert^2+\beta\eta_{t}^{2}\delta
\\ & = (1-\eta_t)f(\X_{s,t})+\eta_t f(\tilde{\X})-\frac{\alpha\eta_t}{2}\Vert \tilde{\X}-\X_{s,t}\Vert^2+\beta\eta_t^{2}\Vert \tilde{\X}-\X_{s,t}\Vert^2 \\ &\ \ \ +\frac{\sigma_{s,t}^2}{2\beta}+\beta\eta_{t}^{2}\delta.
\end{split}
\end{equation*}

Using our assuming that $2\beta\eta_t \le \alpha$ we have that
\begin{equation*}
\begin{split}
\mathbb{E}_{t}[f(\X_{s,t+1})] & \le (1-\eta_t)f(\X_{s,t})+\eta_t f(\tilde{\X})+\frac{\sigma_{s,t}^2}{2\beta}+\beta\eta_{t}^{2}\delta.
\end{split}
\end{equation*}

Taking expectation over both sides w.r.t all randomness, 
we get
\begin{equation}
\begin{split}
\mathbb{E}[f(\X_{s,t+1})] & \le \left(1-\eta_t\right)\mathbb{E}[f(\X_{s,t})]+\eta_t f(\tilde{\X})+\frac{\sigma_{s,t}^2}{2\beta}+\beta\eta_{t}^{2}\delta,
\end{split}
\end{equation}

\end{proof}

\begin{corollary}\label{corr:errDec}
Assume that \cref{Ass1} holds. Fix some epoch $s$ of Algorithm \ref{alg:stoch}, and let $\{\X_{s,t}\}_{t=1}^{T+1}$ be the iterates generated throughout the epoch. Then, if $2\beta\eta_t \le \alpha$, we have that 
\begin{equation} \label{eq:t_rate_full}
\begin{split}
\mathbb{E}[f(\X_{s,t+1})]-f(\X^*) & \le \left(1-\eta_t\right)(\mathbb{E}[f(\X_{s,t})]-f(\X^*))+\frac{\sigma_{s,t}^2}{2\beta}+\beta\eta_{t}^{2}\delta,
\end{split}
\end{equation}
where $\sigma_{s,t}=\sqrt{\mathbb{E}[\Vert \nabla G(\X_{s,t})-\hat{\nabla}g(\X_{s,t})\Vert^2]}$.
\end{corollary}

\begin{proof}

By choosing $\tilde{\X}=\X^*$ in (\ref{eq:t_rate}) and subtracting $f(\X^*)$ from both sides, we get the desired result.

\end{proof}

The following lemma bounds the variance the gradient estimator used in any iteration of the inner-loop of Algorithm \ref{alg:stoch}.

\begin{lemma}[variance bound] \label{lemma:bound}
Assume that \cref{Ass1} holds. Fix some epoch $s$ of Algorithm \ref{alg:stoch}, and let $\{\X_{s,t}\}_{t=1}^{T+1}$ be the iterates generated throughout the epoch. Then,
\begin{eqnarray} \label{eq:bound}
\sigma_{s,t}^2 = \mathbb{E}[\Vert \nabla G(\X_{s,t})-\hat{\nabla}g(\X_{s,t})\Vert^2] &\le& \frac{8\beta_G^2}{\alpha k_t}(\mathbb{E}[f(\X_{s})]-f(\X^*)) \nonumber \\
&&+\frac{8\beta_G^2}{\alpha k_t}(\mathbb{E}[f(\X_{s,t})]-f(\X^*))+\frac{2\sigma^2}{k_s}.
\end{eqnarray}
\end{lemma}

\begin{proof}
Fix some epoch $s$ and iteration $t$ of the inner loop.
Since for all $1\le i<j\le k_t$, $\nabla g^{(i)}(\X)$ and $\nabla g^{(j)}(\X)$ are i.i.d. random variables, and $\mathbb{E}_i[\nabla g^{(i)}(\X)]=\mathbb{E}_j[\nabla g^{(j)}(\X)]=\nabla G(\X)$,
\begin{equation} \label{eq:hatBound}
\begin{split}
& \mathbb{E}\Big[\Big\Vert \frac{1}{k_t}\sum_{i=1}^{k_t}\left(\nabla g^{(i)}(\X_s)-\nabla g^{(i)}(\X_{s,t})\right)-\big(\nabla G(\X_s)-\nabla G(\X_{s,t})\big)\Big\Vert^2\Big] 
\\ & = \frac{1}{k_t}\mathbb{E}\Big[\Big\Vert \nabla g^{(1)}(\X_s)-\nabla g^{(1)}(\X_{s,t})-\big(\nabla G(\X_s)-\nabla G(\X_{s,t})\big)\Big\Vert^2\Big].
\end{split}
\end{equation}

In the same way,
\begin{equation} \label{eq:tildeBound}
\begin{split}
\mathbb{E}[\Vert \nabla G(\X_s)-\tilde{\nabla}g(\X_s)\Vert^2] 
 = \frac{1}{k_t}\mathbb{E}[\Vert \nabla G(\X_s)-\nabla g^{(1)}(\X_s)\Vert^2] \le \frac{\sigma^2}{k_s}.
\end{split}
\end{equation}

By the definition of $\hat{\nabla}g(\X)$ we have that
\begin{align*}
\mathbb{E}[\Vert \nabla G(\X_{s,t})-\hat{\nabla}g(\X_{s,t})\Vert^2]  &= \mathbb{E}\Big[\Big\Vert \nabla G(\X_{s,t})-\nabla G(\X_s)-\frac{1}{k_t}\sum_{i=1}^{k_t}\hat{\nabla}g^{(i)}(\X_{s,t}) \\
&+\frac{1}{k_t}\sum_{i=1}^{k_t}\hat{\nabla}g^{(i)}(\X_s)-\tilde{\nabla}g(\X_s)+\nabla G(\X_s)\Big\Vert^2\Big]\\
& \le  2\mathbb{E}\Big[\Big\Vert \frac{1}{k_t}\sum_{i=1}^{k_t}\left(\hat{\nabla}g^{(i)}(\X_s)-\hat{\nabla}g^{(i)}(\X_{s,t})\right)\\
&-\big(\nabla G(\X_s)-\nabla G(\X_{s,t})\big)\Big\Vert^2\Big]  +2\mathbb{E}[\Vert \nabla G(\X_s)-\tilde{\nabla}g(\X_s)\Vert^2].
\end{align*}


Using (\ref{eq:hatBound}) and (\ref{eq:tildeBound}), we get
\begin{align*}
\mathbb{E}[\Vert \nabla G(\X_{s,t})-\hat{\nabla}g(\X_{s,t})\Vert^2] \le& \frac{2}{k_t}\mathbb{E}[\Vert \nabla g^{(1)}(\X_s)-\nabla g^{(1)}(\X_{s,t})\\
&-(\nabla G(\X_s)-\nabla G(\X_{s,t}))\Vert^2]+\frac{2\sigma^2}{k_s}.
\end{align*}

For any random vector $\v$, the variance is bounded by its second moment, i.e. $\mathbb{E}[\Vert \v-\mathbb{E}[\v]\Vert^2]\le\mathbb{E}[\Vert \v\Vert^2]$. In our case $\mathbb{E}[ \nabla g^{(1)}(\X_s)-\nabla g^{(1)}(\X_{s,t})]=\nabla G(\X_s)-\nabla G(\X_{s,t})$. Therefore,
\begin{equation*}
\begin{split}
 \mathbb{E}[\Vert \nabla G(\X_{s,t})-\hat{\nabla}g(\X_{s,t})\Vert^2] &\le \frac{2}{k_t}\mathbb{E}[\Vert \nabla g^{(1)}(\X_s)-\nabla g^{(1)}(\X_{s,t})\Vert^2]+\frac{2\sigma^2}{k_s}
\\ & = \frac{2}{k_t}\mathbb{E}[\Vert \nabla g^{(1)}(\X_s)-\nabla g^{(1)}(\X^*)\\
&-\nabla g^{(1)}(\X_{s,t})+\nabla g^{(1)}(\X^*)\Vert^2]+\frac{2\sigma^2}{k_s}
\\ & \le \frac{4}{k_t}\mathbb{E}[\Vert \nabla g^{(1)}(\X_s)-\nabla g^{(1)}(\X^*)\Vert^2]\\
&+\frac{4}{k_t}\mathbb{E}[\Vert\nabla g^{(1)}(\X_{s,t})-\nabla g^{(1)}(\X^*)\Vert^2]+\frac{2\sigma^2}{k_s}.
\end{split}
\end{equation*}

Using the $\beta_G$-smoothness of $g^{(1)}$ we have
\begin{equation*}
\begin{split}
\mathbb{E}[\Vert \nabla G(\X_{s,t})-\hat{\nabla}g(\X_{s,t})\Vert^2] & \le \frac{4\beta_G^2}{k_t}\E[\Vert\X_s-\X^*\Vert^2]+\frac{4\beta_G^2}{k_t}\E[\Vert\X_{s,t}-\X^*\Vert^2]+\frac{2\sigma^2}{k_s}.
\end{split}
\end{equation*}

Finally, using the $\alpha$-strong convexity of $f$ we obtain
\begin{equation*}
\begin{split}
 \mathbb{E}[\Vert \nabla G(\X_{s,t})-\hat{\nabla}g(\X_{s,t})\Vert^2] \le &\frac{8\beta_G^2}{\alpha k_t}(\mathbb{E}[f(\X_{s})]-f(\X^*))\\
&+\frac{8\beta_G^2}{\alpha k_t}(\mathbb{E}[f(\X_{s,t})]-f(\X^*))+\frac{2\sigma^2}{k_s}.
\end{split}
\end{equation*}

\end{proof}

The following theorem bounds the approximation error of Algorithm \ref{alg:stoch}.

\begin{theorem} \label{thm:1}
Assume that \cref{Ass1} holds.  Let $\{\X_{s}\}_{s\ge1}$ be a sequence generated by Algorithm \ref{alg:stoch} with parameters $T=\frac{8\beta}{3\alpha}\ln{8}+1$, $\eta_t=\frac{\alpha}{2\beta}$, $k_s=\frac{32\sigma^2}{\alpha C_{0}}2^{s-1}$ and $k_t=\frac{32\beta_{G}^2}{\alpha^2}$, where $C_{0}\ge h_1$. Then,  for all $s\ge1$ it holds that:
\begin{equation} \label{eq:s_rate}
\begin{split}
\mathbb{E}[f(\X_{s})]-f(\X^*) & \le C_{0}\left(\frac{1}{2}\right)^{s-1}+\frac{8\alpha\delta}{7}.
\end{split}
\end{equation}
\end{theorem}

\begin{proof}
Let us define $h_s := \E[f(\X_s)] - f(\X^*)$ for all $s\geq 1$, and $h_{s,t} := \E[f(\X_{s,t})]-f(\X^*)$ for all $s,t\geq 1$. Fix some epoch $s$ and iteration $t$ of the inner loop.

Using Corollary \ref{corr:errDec} and Lemma \ref{lemma:bound} we have that
\begin{equation*}
\begin{split}
h_{s,t+1} & \le \left(1-\eta_t\right)h_{s,t}+\frac{1}{2\beta}\left(\frac{8\beta_G^2}{\alpha k_t}h_s+\frac{8\beta_G^2}{\alpha k_t}h_{s,t}+\frac{2\sigma^2}{k_s}\right)+\beta\eta_{t}^{2}\delta
\\ & = \left(1-\eta_t+\frac{4\beta_G^2}{\alpha\beta k_t} \right)h_{s,t}+\left(\frac{4\beta_G^2}{\alpha\beta k_t}h_s+\frac{\sigma^2}{\beta k_s}+\beta\eta_{t}^{2}\delta\right).
\end{split}
\end{equation*}

Plugging $k_t=\frac{16\beta_G^2}{\alpha\beta\eta_t}$ we get
\begin{equation*}
\begin{split}
h_{s,t+1} & \le \left(1-\eta_t+\frac{\eta_t}{4} \right)h_{s,t}+\left(\frac{\eta_t}{4}h_s+\frac{\sigma^2}{\beta k_s}+\beta\eta_{t}^{2}\delta\right).
\end{split}
\end{equation*}

Plugging $\eta_t=\frac{\alpha}{2\beta}$ we get
\begin{equation*}
\begin{split}
h_{s,t+1} & \le \left(1-\frac{3\alpha}{8\beta}\right)h_{s,t}+\left(\frac{\alpha}{8\beta}h_s+\frac{\sigma^2}{\beta k_s}+\frac{\alpha^{2}\delta}{4\beta}\right).
\end{split}
\end{equation*}

Fixing an epoch $s$ and unrolling the recursion for $t= (T-1)\dots{}1$ we get 
\begin{equation*}
\begin{split}
h_{s,T} & \le \left(1-\frac{3\alpha}{8\beta}\right)h_{s,T-1}+\left(\frac{\alpha}{8\beta}h_s+\frac{\sigma^2}{\beta k_s}+\frac{\alpha^{2}\delta}{4\beta}\right)
\\ & \le ... \le \left(1-\frac{3\alpha}{8\beta}\right)^{T-1}h_{s,1}+\left(\frac{\alpha}{8\beta}h_s+\frac{\sigma^2}{\beta k_s}+\frac{\alpha^{2}\delta}{4\beta}\right)\sum_{k=1}^{T-1}\left(1-\frac{3\alpha}{8\beta}\right)^{T-k-1}
\\ & = \left(1-\frac{3\alpha}{8\beta}\right)^{T-1}h_{s,1}+\left(\frac{1}{3}h_s+\frac{8\sigma^2}{3\alpha k_s}+\frac{2\alpha\delta}{3}\right)\left(1-\left(1-\frac{3\alpha}{8\beta}\right)^{T-1}\right).
\end{split}
\end{equation*}

$h_{s,T}=h_{s+1}$ and $h_{s,1}=h_{s}$ and so
\begin{equation*}
\begin{split}
h_{s+1} & \le \left(1-\frac{3\alpha}{8\beta}\right)^{T-1}h_{s}+\left(\frac{1}{3}h_s+\frac{8\sigma^2}{3\alpha k_s}+\frac{2\alpha\delta}{3}\right)\left(1-\left(1-\frac{3\alpha}{8\beta}\right)^{T-1}\right)
\\ & = \left(\frac{1}{3}+\frac{2}{3}\left(1-\frac{3\alpha}{8\beta}\right)^{T-1}\right)h_{s}+\left(\frac{8\sigma^2}{3\alpha k_s}+\frac{2\alpha\delta}{3}\right)\left(1-\left(1-\frac{3\alpha}{8\beta}\right)^{T-1}\right)
\\ & \le \left(\frac{1}{3}+\frac{2}{3}e^{-\frac{3\alpha}{8\beta}(T-1)}\right)h_{s}+\left(\frac{8\sigma^2}{3\alpha k_s}+\frac{2\alpha\delta}{3}\right)\left(1-\left(1-\frac{3\alpha}{8\beta}\right)^{T-1}\right).
\end{split}
\end{equation*}

Choosing $T=\frac{8\beta}{3\alpha}\ln{8}+1$, we get
\begin{equation} \label{eq:2}
\begin{split}
h_{s+1} & \le \left(\frac{1}{3}+\frac{2}{3}e^{-\frac{3\alpha}{8\beta}(\frac{8\beta}{3\alpha}\ln{8})}\right)h_{s}+\left(\frac{8\sigma^2}{3\alpha k_s}+\frac{2\alpha\delta}{3}\right)\left(1-\left(1-\frac{3\alpha}{8\beta}\right)^{\frac{8\beta}{3\alpha}\ln{8}}\right)
\\ & = \frac{5}{12}h_{s}+\left(\frac{8\sigma^2}{3\alpha k_s}+\frac{2\alpha\delta}{3}\right)\left(1-\left(1-\frac{3\alpha}{8\beta}\right)^{\frac{8\beta}{3\alpha}\ln{8}}\right)
\\ & \le \frac{5}{12}h_{s}+\frac{8\sigma^2}{3\alpha k_s}+\frac{2\alpha\delta}{3}.
\end{split}
\end{equation}

Now, we use induction over $s$ to prove our claimed bound 
\begin{equation} \label{eq:inductionHypnosis}
h_s \le C_{0}\left(\frac{1}{2}\right)^{s-1}+\frac{8\alpha\delta}{7}. 
\end{equation}

The base case $s=1$, follows from the choice $C_{0}\ge h_1$.

For $s\ge1$ using (\ref{eq:2}) with $k_s=\frac{32\sigma^2}{\alpha C_{0}}2^{s-1}$ we get,
\begin{equation*}
\begin{split}
h_{s+1} & \le \frac{5}{12}h_s+\frac{C_0}{12}\left(\frac{1}{2}\right)^{s-1}+\frac{2\alpha\delta}{3}.
\end{split}
\end{equation*}

Using the induction hypothesis for $h_s$ in (\ref{eq:inductionHypnosis}) gives us
\begin{equation*}
\begin{split}
h_{s+1} & \le \frac{5}{12}C_{0}\left(\frac{1}{2}\right)^{s-1}+\frac{5}{12}\frac{8\alpha\delta}{7}+\frac{C_{0}}{12}\left(\frac{1}{2}\right)^{s-1}+\frac{2\alpha\delta}{3} = C_{0}\left(\frac{1}{2}\right)^{s}+\frac{8\alpha\delta}{7}.
\end{split}
\end{equation*}

\end{proof}


We now prove \cref{thm:stocohasticRates}, which is a direct corollary of \cref{thm:1}. 

\begin{proof}[Proof of \cref{thm:stocohasticRates}]

By \cref{thm:1} it is implied that to achieve an $\varepsilon$-expected error, it suffices to fix $\delta = \frac{7\epsilon}{16\alpha}$ and to complete
\begin{equation*}
\begin{split}
S & = \log_2{\left(\frac{C_0}{\varepsilon}\right)}+2
\end{split}
\end{equation*}
epochs of Algorithm \ref{alg:stoch}.

For this number of epochs we upper bound the overall number of stochastic gradients as follows.
\begin{equation} \label{num_ks}
\begin{split}
\sum_{s=1}^{S}k_s & = \frac{32\sigma^2}{\alpha C_0}\sum_{s=1}^{S}2^{s-1}
 = \frac{32\sigma^2}{\alpha C_0}\left(2^S-1\right)
 = \frac{32\sigma^2}{\alpha C_0}\left(2^{\log_2{\left(\frac{C_0}{\varepsilon}\right)}+2}-1\right)
 \leq \frac{128\sigma^2}{\alpha}\frac{1}{\epsilon}.
\end{split}
\end{equation}

\begin{equation} \label{num_kt}
\begin{split}
\sum_{s=1}^{S}\sum_{t=1}^{T}k_t & = \sum_{s=1}^{S}\sum_{t=1}^{T}\frac{32\beta_G^2}{\alpha^2}
 = \frac{32\beta_G^2}{\alpha^2}\left(\frac{8\beta}{3\alpha}\ln{8}+1\right)\left(\log_2{\left(\frac{C_0}{\varepsilon}\right)}+2\right).
\end{split}
\end{equation}

All together,
\begin{equation} \label{num_ks_kt}
\begin{split}
\sum_{s=1}^{S}k_s+\sum_{s=1}^{S}\sum_{t=1}^{T}k_t & \leq \frac{128\sigma^2}{\alpha}\frac{1}{\epsilon}+ \frac{32\beta_G^2}{\alpha^2}\left(\frac{8\beta}{3\alpha}\ln{8}+1\right)\left(\log_2{\left(\frac{C_0}{\varepsilon}\right)}+2\right).
\end{split}
\end{equation} 

\end{proof}

\subsection{Finite-sum setting} \label{sec:finiteSum}
In this section we assume that $G(\X)$ from Problem (\ref{eq:generalModel}) is in the form of a finite sum, i.e.
\[ G(\X)=\frac{1}{n}\sum_{i=1}^{n}g_{i}(\X). \]

The stochastic oracle in this setting simply samples a function $g_i(\X)$, $i\in[n]$, uniformly at random. In this case, in the outer loop of \cref{alg:stoch} we take \[ \tilde{\nabla}g(\X)=\frac{1}{n}\sum_{i=1}^{n}\nabla g_{i}(\X)=\nabla G(\X). \]

\begin{algorithm}[H]
	\caption{Finite-Sum Variance-Reduced Generalized Conditional Gradient}
	\label{alg:stochFiniteSum}
	\begin{algorithmic}
		\STATE  \textbf{Input:} $T$, $\{\eta_t\}_{t=1}^{T-1}\subset[0,1]$, $\{k_t\}_{t=1}^{T-1}, \{k_s\}_{s\geq 1}\subset\mathbb{N}$, $\delta\ge0$.
		\STATE \textbf{Initialization:} Choose some $\X_{1}\in dom(h)$.		 
		\FOR {$s=1,2,...$}	
		\STATE $\tilde{\nabla}g(\X_s)=\frac{1}{n}\sum_{i=1}^{n}\nabla g_{i}(\X_s)$ \COMMENT{snap-shot gradient}.
		\STATE $\X_{s,1}=\X_{s}$
		\FOR{$t = 1,2,...,T-1$} 
			\STATE Sample $g^{(1)},...,g^{(k_t)}$ from $\mathcal{D}$.
			\STATE Define $\hat{\nabla}g(\X_{s,t})=\frac{1}{k_t}\sum_{i=1}^{k_t}\left(\nabla g^{(i)}(\X_{s,t})-\left(\nabla g^{(i)}(\X_{s})-\tilde{\nabla}g(\X_s)\right)\right)$.
			\STATE $\V_{t}=\argmin\limits_{\V \in \mathbb{V}} \left\{ \psi_t(\V):= \Vert \V-\X_{s,t}+\frac{1}{2\beta\eta_t}(\hat{\nabla}g(\X_{s,t})+\nabla R(\X_{s,t}))\Vert^{2}_{F}+\frac{1}{\beta\eta_t}h(\V)\right\}$ 
			\\ \COMMENT{in fact it suffices that $\psi_t(\V_t) \leq \psi_t(\X^*)+\delta$ for some optimal solution $\X^*$}.
            \STATE $ \X_{s,t+1}=(1-\eta_t)\X_{s,t}+\eta_t \V_t$
         \ENDFOR
         \STATE $\X_{s+1}=\X_{s,T}$
         \ENDFOR
	\end{algorithmic}
\end{algorithm}

The following theorem is analogous to Theorem \ref{thm:1} and bounds the approximation error of Algorithm \ref{alg:stochFiniteSum}.

\begin{theorem} \label{thm:1FiniteSum}
Assume that \cref{Ass1} holds.  Let $\{\X_{s}\}_{s\ge1}$ be a sequence generated by Algorithm \ref{alg:stochFiniteSum}. Then, Algorithm \ref{alg:stochFiniteSum} with $T=\frac{8\beta}{3\alpha}\ln{8}+1$ iterations of the inner loop at each epoch $s$, a step size of $\eta_t=\frac{\alpha}{2\beta}$, and $k_t=\frac{32\beta_{G}^2}{\alpha^2}$ gradients implemented by the stochastic oracle at inner loop iterations $t$, such that $C_{0}\ge h_1$, for all $s\ge1$ guarantees that:
\begin{equation*}
\begin{split}
\mathbb{E}[f(\X_{s})]-f(\X^*) & \le C_{0}\left(\frac{5}{12}\right)^{s-1}+\frac{8\alpha\delta}{7}.
\end{split}
\end{equation*}
\end{theorem}

\begin{proof}

Since $\tilde{\nabla}g(\X)=\frac{1}{n}\sum_{i=1}^{n}\nabla g_{i}(\X)=\nabla G(\X)$, we get that
$\mathbb{E}[\Vert \nabla G(\X_s)-\tilde{\nabla}g(\X_s)\Vert^2]=0$.
Using this inequality instead of \eqref{eq:tildeBound} in the proof of Lemma \ref{lemma:bound}, directly gives us the improved bound:
\begin{equation*}
\begin{split}
\mathbb{E}[\Vert \nabla G(\X_{s,t})-\hat{\nabla}g(\X_{s,t})\Vert^2] \le \frac{8\beta_G^2}{\alpha k_t}(\mathbb{E}[f(\X_{s})]-f(\X^*))+\frac{8\beta_G^2}{\alpha k_t}(\mathbb{E}[f(\X_{s,t})]-f(\X^*)).
\end{split}
\end{equation*}

We define $h_s, h_{s,t}$ for all $s,t\geq 0$ as in the proof of Theorem \ref{thm:1}.

Plugging the above new bound into Corollary \ref{corr:errDec}, we get
\begin{equation*}
\begin{split}
h_{s,t+1} & \le \left(1-\eta_t \right)h_{s,t}+\frac{1}{2\beta}\left(\frac{8\beta_G^2}{\alpha k_t}h_s+\frac{8\beta_G^2}{\alpha k_t}h_{s,t}\right)+\beta\eta_t^{2}\delta
\\ & = \left(1-\eta_t+\frac{4\beta_G^2}{\alpha\beta k_t} \right)h_{s,t}+\frac{4\beta_G^2}{\alpha\beta k_t}h_s+\beta\eta_t^{2}\delta.
\end{split}
\end{equation*}

From here the rest of the proof closely follows that of Theorem \ref{thm:1}.

Taking $k_t=\frac{16\beta_G^2}{\alpha\beta\eta_t}$,
\begin{equation*}
\begin{split}
h_{s,t+1} & \le \left(1-\eta_t+\frac{\eta_t}{4} \right)h_{s,t}+\frac{\eta_t}{4}h_s+\beta\eta_t^{2}\delta.
\end{split}
\end{equation*}

Taking $\eta_t=\frac{\alpha}{2\beta}$ we get
\begin{equation*}
\begin{split}
h_{s,t+1} & \le \left(1-\frac{3\alpha}{8\beta}\right)h_{s,t}+\frac{\alpha}{8\beta}h_s+\frac{\alpha^2\delta}{4\beta}.
\end{split}
\end{equation*}

Unrolling the recursion for all $t$ in epoch $s$:
\begin{equation*}
\begin{split}
h_{s,T} & \le \left(1-\frac{3\alpha}{8\beta}\right)h_{s,T-1}+\frac{\alpha}{8\beta}h_s+\frac{\alpha^2\delta}{4\beta}
\\ & \le ... \le \left(1-\frac{3\alpha}{8\beta}\right)^{T-1}h_{s,1}+\left(\frac{\alpha}{8\beta}h_s+\frac{\alpha^2\delta}{4\beta}\right)\sum_{k=1}^{T-1}\left(1-\frac{3\alpha}{8\beta}\right)^{T-k-1}
\\ & = \left(1-\frac{3\alpha}{8\beta}\right)^{T-1}h_{s,1}+\left(\frac{1}{3}h_s+\frac{2\alpha\delta}{3}\right)\left(1-\left(1-\frac{3\alpha}{8\beta}\right)^{T-1}\right).
\end{split}
\end{equation*}

$h_{s,T}=h_{s+1}$ and $h_{s,1}=h_{s}$ and so
\begin{equation*}
\begin{split}
h_{s+1} & \le \left(1-\frac{3\alpha}{8\beta}\right)^{T-1}h_{s}+\left(\frac{1}{3}h_s+\frac{2\alpha\delta}{3}\right)\left(1-\left(1-\frac{3\alpha}{8\beta}\right)^{T-1}\right)
\\ & = \left(\frac{1}{3}+\frac{2}{3}\left(1-\frac{3\alpha}{8\beta}\right)^{T-1}\right)h_{s}+\frac{2\alpha\delta}{3}
\\ & \le \left(\frac{1}{3}+\frac{2}{3}e^{-\frac{3\alpha}{8\beta}(T-1)}\right)h_{s}+\frac{2\alpha\delta}{3}.
\end{split}
\end{equation*}

Choosing $T=\frac{8\beta}{3\alpha}\ln{8}+1$, we get
\begin{equation*}
\begin{split}
h_{s+1} & \le \left(\frac{1}{3}+\frac{2}{3}e^{-\frac{3\alpha}{8\beta}(\frac{8\beta}{3\alpha}\ln{8})}\right)h_{s}+\frac{2\alpha\delta}{3} = \frac{5}{12}h_{s}+\frac{2\alpha\delta}{3}.
\end{split}
\end{equation*}

Using the same induction argument as in the proof of Theorem \ref{thm:1}, we conclude that for all $s$:
\begin{equation*}
\begin{split}
h_{s+1} & \le \left(\frac{5}{12}\right)^s C_{0}+\frac{2\alpha\delta}{3}\sum_{k=1}^{s}\left(\frac{5}{12}\right)^{s-k}=\left(\frac{5}{12}\right)^s C_{0}+\frac{24\alpha\delta}{21}\left(1-\left(\frac{5}{12}\right)^{s}\right)
\\ & \le\left(\frac{5}{12}\right)^s C_{0}+\frac{8\alpha\delta}{7}.
\end{split}
\end{equation*}

\end{proof}


We now prove \cref{thm:finiteSumRates}, which is a direct corollary of \cref{thm:1FiniteSum}. 

\begin{proof}[Proof of \cref{thm:finiteSumRates}]

By \cref{thm:1FiniteSum} it is implied that to achieve an $\varepsilon$-expected error, setting $\delta = \frac{7\epsilon}{16\alpha}$, we need to compute at most
\begin{equation*}
\begin{split}
S & = \log_{\frac{12}{5}}{\left(\frac{2C_0}{\varepsilon}\right)}+1
\end{split}
\end{equation*}
epochs of Algorithm \ref{alg:stochFiniteSum}.

Therefore, the overall number of exact gradients to be computes is at most
\begin{equation*}
\begin{split}
\sum_{s=1}^{S}n & = n\left(\log_{\frac{12}{5}}{\left(\frac{2C_0}{\varepsilon}\right)}+1\right),
\end{split}
\end{equation*}

and the overall number of stochastic gradients is at most
\begin{equation*}
\begin{split}
\sum_{s=1}^{S}\sum_{t=1}^{T}k_t & = \sum_{s=1}^{S}\sum_{t=1}^{T}\frac{32\beta_G^2}{\alpha^2}
 = \frac{32\beta_G^2}{\alpha^2}\left(\frac{8\beta}{3\alpha}\ln{8}+1\right)\left(\log_{\frac{12}{5}}{\left(\frac{2C_0}{\varepsilon}\right)}+1\right).
\end{split}
\end{equation*}


\end{proof}

\section{Applications to Non-smooth Problems} \label{sec:smoothing}

In this section we turn to discuss applications of our results to non-smooth problems. Concretely, we consider composite models which take the form of Model \eqref{eq:generalModel}, with the difference that we now assume that the function $R(\X)$ is \textit{non-smooth}, however, admits a known \textit{smoothing} scheme. We then discuss in detail three concrete applications of interest: recovering a simultaneously low-rank and sparse matrix, recovering a low-rank matrix subject to linear constraints, and recovering a low-rank and sparse matrix from linear measurements with the elastic-net regularizer.

\subsection{Applying our results to non-smooth problems via smoothing}

In order to fit the nonsmooth problems considered in this section to our smooth model \eqref{eq:generalModel}, we build on the smoothing framework introduced in \cite{smoothing}, which replaces the nonsmooth term $R(\X)$ with a smooth approximation.

The following definition is taken from \cite{smoothing}.
\begin{definition}
Let $R:\mathbb{V}\rightarrow(-\infty,\infty]$ be a closed, proper and convex function and let $X\subseteq dom(R)$ be a closed and convex set. $R$ is $(\theta,\gamma,K)$-smoothable over $X$ if there exists $\gamma_{1}$ and $\gamma_{2}$ such that $\gamma=\gamma_{1}+\gamma_{2}\ge0$ such that for every $\mu\ge0$ there exists a continuously differentiable function $R_{\mu}:\mathbb{V}\rightarrow(-\infty,\infty]$ such that:
\renewcommand{\labelenumi}{(\alph{enumi})}
\begin{enumerate}
\item $R(x)-\gamma_{1}\mu\le R_{\mu}(x)\le R(x)+\gamma_{2}\mu$ for every $x\in X$.
\item There exists $K\ge0$ and $\theta\ge0$ such that $\Vert\nabla R_{\mu}(x)-\nabla R_{\mu}(x)\Vert\le\left(K+\frac{\theta}{\mu}\right)\Vert x-y\Vert$ for every $x,y\in X$.
\end{enumerate}
\end{definition}

Formally, now we consider applying our algorithms to non-smooth optimization problems of the following form:
\begin{equation} \label{eq:generalModel_nonsmooth}
\min_{\X\in \mathbb{V}}{f(\X):=G(\X)+R(\X)+h(\X)},
\end{equation}

with the following assumptions (replacing Assumption \ref{Ass1}):
\begin{assumption} \label{Ass2}
\begin{itemize} 
\item $G$ is stochastic, i.e., $G(\X)=\mathbb{E}_{g\sim \mathcal{D}}[g(\X)]$, where $\mathcal{D}$ is a distribution over functions $g:\mathbb{V}\rightarrow\mathbb{R}$, given by a sampling oracle. $G$ is convex and differentiable, and for all $g\in supp(\mathcal{D})$, $g$ is $\beta_{G}$-smooth, and there exists $\sigma\ge0$ such that $\sigma \ge \sup\limits_{\X\in\mathbb{V}}\sqrt{\mathbb{E}[\Vert\nabla G(\X)-\nabla g(\X)\Vert^{2}]}$.
\item $R:\mathbb{V}\rightarrow (-\infty,\infty]$ is deterministic, $(\theta,\gamma,K)$-smoothable, and convex.
\item $G+R$ is $\alpha$-strongly convex. 
\item $h:\mathbb{V}\rightarrow (-\infty,\infty]$ is deterministic, non-smooth, proper, lower semicontinuous and convex.
\end{itemize}
\end{assumption}

We will denote the $\mu$-smooth approximation of $R(\X)$ as $R_{\mu}(\X)$, and its smoothness parameter to be $\beta_R=\left(K+\frac{\theta}{\mu}\right)$. 

As in our discussions so far, considering Model \eqref{eq:generalModel_nonsmooth} especially in the context of low-rank matrix optimization problems (e.g., $h(\cdot)$ is an indicator function for a nuclear-norm ball or the trace-bounded positive semidefinite cone, or an analogous regularization function),
we assume that the optimal solution $\X^*$ is naturally of low-rank and we want to rely on SVD computations whose rank does not exceeds that of $\X^*$ - the optimal solution to the \textit{original non-smooth problem}. However, when put in the context of this section and considering Model \eqref{eq:generalModel_nonsmooth}, the rank of SVD computations required by the results developed in previous sections corresponded to the optimal solution of the \textit{smoothed problem}, i.e., after $R(\cdot)$ is replaced with a smooth approximation $R_{\mu}(\cdot)$. In particular, it can very much be the case, that even though the optimal solution to the smooth problem is very close (both in norm and in function value) to the optimal solution of the non-smooth problem, its rank is much higher. Thus, in this section, towards developing an algorithm that relies on SVD computation with rank at most that of the \textit{non-smooth optimum}, we introduce the following modified definition of a weak-proximal oracle.


\begin{definition} \label{Def:weakProx}
We say an Algorithm $\mA$ is a $(\delta_1,\delta_2)$-weak proximal oracle for Model \eqref{eq:generalModel_nonsmooth},
if for point $\X\in dom(h)$ and step-size $\eta$, $\mA(\X,\eta)$ returns a point $\V\in dom(h)$ such that $\psi(\V,\X,\eta) \leq \psi(\tilde{\X}^*,\X,\eta)+\delta_1$,
where $\tilde{\X}^*$ is a feasible point satisfying $\vert{f(\X^*) - f(\tilde{\X}^*)}\vert \leq \delta_2$, 
\begin{eqnarray*}
\psi(\V,\X,\eta):= \Vert \V-\X+\frac{1}{2\beta\eta_t}(\hat{\nabla}g(\X)+\nabla R_{\mu}(\X))\Vert^{2}+\frac{1}{\beta\eta_t}h(\V),
\end{eqnarray*}
and $R_{\mu}(\cdot)$ is the $\mu$-smooth approximation of $R(\cdot)$.
\end{definition}

Henceforth, we consider Algorithm \ref{alg:stoch} with the single difference: now $\V_t$ is the ouput of a $(\delta_1,\delta_2)$-weak proximal oracle, as defined in Definition \ref{Def:weakProx}.

Note that in the context of low-rank problems and in the ideal case $\delta_1=\delta_2=0$\footnote{these can be made arbitarily small by the choice of smoothing parameter and accuracy in SVD computations.}, the implementation of the oracle in Definition \ref{Def:weakProx} is exactly the same as the weak proximal oracle discussed before, i.e., if $h(\cdot)$ is for instance the indicator function for a radius-$\tau$ nuclear-norm ball, then implementing the oracle in Definition \ref{Def:weakProx} amounts to a Euclidean projection of the $\rank(\X^*)$-approximation of $\A_t := \X-\frac{1}{2\beta\eta_t}(\hat{\nabla}g(\X)+\nabla R_{\mu}(\X))$ onto the nuclear-norm ball. Here, the tolerances $\delta_1,\delta_2$ allow us to absorb the error due to the smoothing approximation and numerical errors in SVD computations.

The following theorem is analogues to Theorem \ref{thm:1}.

\begin{theorem} \label{thm:nonsmooth}
Assume that \cref{Ass2} holds.  Let $\{\X_{s}\}_{s\ge1}$ be a sequence generated by \cref{alg:stoch} when applied to the smooth approximation of Problem (\ref{eq:generalModel_nonsmooth}), and let $\X^*$ denote the optimal solution of the non-smooth problem.
Then, using the parameters  $T=\frac{8\beta}{3\alpha}\ln{8}+1$, $\eta_t=\frac{\alpha}{2\beta}$, $k_s=\frac{32\sigma^2}{\alpha C_{0}}2^{s-1}$ and $k_t=\frac{32\beta_{G}^2}{\alpha^2}$ for $C_0$ such that $C_{0}\ge h_1$, guarantees that for all $s\ge1$:
\begin{equation*} 
\begin{split}
\mathbb{E}[f(\X_{s})]-f(\X^*) & \le C_{0}\left(\frac{1}{2}\right)^{s-1}+\frac{8}{7}\alpha\delta_{1}+\frac{23}{7}\gamma\mu.
\end{split}
\end{equation*}
\end{theorem}

\begin{proof}

Denote the smoothed function by $f_{\mu}(\X):=G(\X)+R_{\mu}(\X)+h(\X)$. Let $\X^*$ and $\X^*_{\mu}$ denote the optimal solutions of the non-smooth and smoothed functions respectively.
By applying \cref{alg:stoch} to $f_{\mu}(\X)$, such that at each iteration $\V_t$ is chosen as a point that satisfies $\psi_{t}(\V_t)\le\psi_{t}(\X^*)+\delta_{1}$, we get according to \cref{lemma:psi_error}
\begin{equation} \label{eq:t_rate_for_smoothed}
\begin{split}
\mathbb{E}[f_{\mu}(\X_{s,t+1})] & \le \left(1-\eta_t\right)\mathbb{E}[f_{\mu}(\X_{s,t})]+\eta_t f_{\mu}(\X^*)+\frac{\sigma_{s,t}^2}{2\beta}+\beta\eta_{t}^{2}\delta_{1}.
\end{split}
\end{equation}

We notice that by the definition of the smoothing and optimality of $\X^*$,
\begin{equation} \label{eq:smoothingInequalities}
\begin{split}
f_{\mu}(\X^*) \le f(\X^*)+\gamma_{2}\mu \le f(\X_{\mu}^*)+\gamma_{2}\mu \le f_{\mu}(\X_{\mu}^*)+\gamma\mu.
\end{split}
\end{equation}

By plugging (\ref{eq:smoothingInequalities}) into (\ref{eq:t_rate_for_smoothed}) and subtracting $f_{\mu}(\X_{\mu}^*)$ from both sides we get
\begin{equation*}
\begin{split}
\mathbb{E}[f_{\mu}(\X_{s,t+1})]-f_{\mu}(\X_{\mu}^*) & \le \left(1-\eta_t\right)(\mathbb{E}[f_{\mu}(\X_{s,t})]-f_{\mu}(\X_{\mu}^*))+\frac{\sigma_{s,t}^2}{2\beta}+\beta\eta_{t}^{2}\delta_{1}+\eta_t\gamma\mu.
\end{split}
\end{equation*}

Following the proof of \cref{thm:1} with $\delta=\delta_{1}+\frac{\mu\gamma}{\beta\eta_t}$ gives us
\begin{equation} \label{eq:s_rate_smoothed}
\begin{split}
\mathbb{E}[f_{\mu}(\X_{s})]-f_{\mu}(\X_{\mu}^*) & \le C_{0}\left(\frac{1}{2}\right)^{s-1}+\frac{8}{7}\left(\alpha\delta_{1}+2\gamma\mu\right).
\end{split}
\end{equation}

Using the optimality of $\X_{\mu}^*$ and the definition of the smoothing we get,
\begin{equation} \label{eq:smoothingInequalities2}
\begin{split}
\mathbb{E}[f_{\mu}(\X_{s})]-f_{\mu}(\X_{\mu}^*) \ge \mathbb{E}[f_{\mu}(\X_{s})]-f_{\mu}(\X_{\mu}^*) \ge \mathbb{E}[f(\X_{s})]-f(\X_{\mu}^*)-\gamma\mu.
\end{split}
\end{equation}

Combining (\ref{eq:s_rate_smoothed}) and (\ref{eq:smoothingInequalities2}) we obtain
\begin{equation*} 
\begin{split}
\mathbb{E}[f(\X_{s})]-f(\X^*) & \le C_{0}\left(\frac{1}{2}\right)^{s-1}+\frac{8}{7}\alpha\delta_{1}+\frac{23}{7}\gamma\mu.
\end{split}
\end{equation*}

\end{proof}

\begin{corollary} \label{cor:nonsmooth}
Assume that \cref{Ass2} holds. Applying Theorem \ref{thm:nonsmooth} with the parameters $\delta_1=\frac{7\varepsilon}{32\alpha}$ and $\mu=\frac{7\varepsilon}{92\gamma}$, guarantees that the overall number of epochs to reach an $\epsilon$-approximated solution in expectation is bounded by
\begin{equation*}
O\left(\ln{\left(\frac{1}{\varepsilon}\right)}\right),
\end{equation*}
the total number of calls to the $(\delta_1,\delta_2)$-weak proximal oracle is bounded by 
\begin{equation*}
O\left(\frac{\beta}{\alpha}\ln\left(\frac{1}{\varepsilon}\right)\right),
\end{equation*}
and the total number of stochastic gradients sampled is bounded by
\begin{equation*}
O\left(\frac{\sigma^2}{\alpha\varepsilon}+\frac{\beta_G^2\beta}{\alpha^3}\ln{\left(\frac{1}{\varepsilon}\right)}\right).
\end{equation*}

\end{corollary}

\begin{proof}
By \cref{thm:nonsmooth} it is implied that to achieve an $\varepsilon$-stochastic error $\mathbb{E}[f(\X_S)]-f(\X^*)\le\varepsilon$ we need to compute
\begin{equation*}
\begin{split}
S & \ge \log_2{\left(\frac{C_0}{\varepsilon}\right)}+2
\end{split}
\end{equation*}
iterations.

The rest follows from the calculations brought in (\ref{num_ks}),(\ref{num_kt}),(\ref{num_ks_kt}).
\end{proof}

\subsection{Specific examples}

We now discuss several applications of Corollary \ref{cor:nonsmooth} to specific problems.

\subsubsection{Example 1: Low-rank and sparse matrix estimation}

As discussed in the introduction, this work is largely motivated by matrix recovery problems, such as low-rank and sparse matrix estimation. In order to show the application of our algorithm for this matrix estimation problem, we state a corresponding optimization problem:
\begin{equation} \label{eq:matrixEstimation}
\min_{\Vert \X\Vert_{*}\le\tau}{\frac{1}{2}\Vert \X-\mathbb{E}_{\M\sim \mathcal{D}}[\M]\Vert_{F}^{2}+\lambda\Vert \X\Vert_{1}},
\end{equation}
where $\mathcal{D}$ is an unknown distribution over instances.


For problem (\ref{eq:matrixEstimation}) to fit the Model \eqref{eq:generalModel_nonsmooth}, we take
\begin{equation} \label{eq:G(x)}
G(\X)=\mathbb{E}_{(\M,\N)\sim \mathcal{D}\times\mathcal{D}}\left[\frac{1}{2}\langle{\X-\M,\X-\N}\rangle\right].
\end{equation}

Since $\M$ and $\N$ are i.i.d, this is equivalent to
\begin{equation*}
\begin{split}
G(\X) & = \frac{1}{2}\langle{\mathbb{E}_{\M\sim\mathcal{D}}[\X-\M],\mathbb{E}_{\N\sim\mathcal{D}}[\X-\N]}\rangle
\\ & = \frac{1}{2}\langle{\X-\mathbb{E}_{\M\sim \mathcal{D}}[\M],\X-\mathbb{E}_{\M\sim \mathcal{D}}[\M]}\rangle
\\ & = \frac{1}{2}\Vert \X-\mathbb{E}_{\M\sim \mathcal{D}}[\M]\Vert_{F}^{2}.
\end{split}
\end{equation*}

It should be noted that for this function $G(\X)$, the stochastic gradients are of the form $\nabla g^{(i)}(\X)=\X-\M$ for some $\M\sim\mathcal{D}$. As a result, for a fixed epoch $s$ and iteration $t$, we have
\[ \hat{\nabla}g(\X_{s,t})=\frac{1}{k_t}\sum_{i=1}^{k_t}\left(\X_{s,t}-\X_{s}+\tilde{\nabla}g(\X_s)\right). \]
As can be seen, $\hat{\nabla}g(\X_{s,t})$ is independent of the stochastic samples within the inner-loop (since they cancel-out), and therefore we can simple set $k_t=0$.

Smoothing the $\ell_1$-norm has a well known solution, as shown in \cite{smoothing}. The $\mu$-smooth approximation of $\Vert \X\Vert_{1}$ is 
\[ R_{\mu}(\X)=\sum_{j=1}^{d}\sum_{i=1}^{m}H_{\mu}( \X_{ij}), \]
with parameters $(1,\frac{md}{2},0)$, where $H_{\mu}(t)$ is the one dimensional Huber function, defined as:
\begin{equation*}
\begin{split}
H_{\mu}(t) & =\begin{cases}
\frac{t^{2}}{2\mu}, & |t|\le\mu\\
|t|-\frac{\mu}{2}, & |t|>\mu
\end{cases}.
\end{split}
\end{equation*}
This satisfies
\begin{equation} \label{eq:l1HuberBound}
\begin{split} 
R_{\mu}(\X) \le \Vert \X\Vert_{1} \le R_{\mu}(\X)+\frac{md\mu}{2}.
\end{split}
\end{equation}

$h(\X)$ is to be taken to be the indicator over the nuclear norm ball, i.e., $h(\X)=\chi_\mathcal{C}$, where $\mathcal{C}=\{\X\in\mathbb{R}^{m\times d}:\ \Vert \X\Vert_{*}\le\tau\}$.



\begin{corollary} \label{cor:genEx1}
Consider running Algorithm \ref{alg:stoch} for the smooth approximation of Problem (\ref{eq:matrixEstimation}), with parameters $T=\frac{8\ln{8}}{3}\left(\frac{\lambda}{\mu}+1\right)+1$, $\eta_t=\frac{\mu}{2\lambda+2\mu}$, $k_s=\frac{32\mathbb{E}[\Vert M-\mathbb{E}[M]\Vert^{2}]}{C_{0}}2^{s-1}$ and $k_t=32$. Let $\X^*$ denote the optimal solution Problem \eqref{eq:matrixEstimation}. Then, running $S \ge \log_2{\left(\frac{C_0}{\varepsilon}\right)}+2$ epochs of the outer-loop guarantees that:
\begin{equation} \label{eq:MEsmoothing}
\begin{split}
\mathbb{E}[f(\X_{S})]-f(\X^*) & \le \frac{\varepsilon}{4}+\frac{8}{7}\delta_{1}+\frac{23md}{14}\mu.
\end{split}
\end{equation}
In particular, taking a smoothing parameter of $\mu=\frac{7\varepsilon}{46md}$ and $\delta_1=\frac{7\varepsilon}{32}$, we obtain
\begin{equation} \label{eq:MEsmoothing_particular}
\begin{split}
\mathbb{E}[f(\X_{S})]-f(\X^*) & \le \varepsilon.
\end{split}
\end{equation}
\end{corollary}

\begin{proof}
The parameters of the problem are as follows:
$\alpha=1$, $\beta_G=1$, $\gamma=\frac{md}{2}$, $\beta_R=\frac{\lambda}{\mu}$, $\sigma^2=\mathbb{E}[\Vert M-\mathbb{E}[M]\Vert^{2}]$.
Therefore, by \cref{thm:nonsmooth} we get the result in (\ref{eq:MEsmoothing}).
By choosing $\mu=\frac{7\varepsilon}{46md}$ and $\delta_1=\frac{7\varepsilon}{32}$, the result in (\ref{eq:MEsmoothing_particular}) is immediate. 
\end{proof}

\subsubsection{Example 2: Linearly constrained low-rank matrix estimation}

Another example, is the problem of recovering a low-rank matrix subject to linear constraints, which can be written in penalized form as:
\begin{equation} \label{eq:linearMax}
\min_{\Vert \X\Vert_{*}\le\tau}{\frac{1}{2}\Vert \X-\mathbb{E}_{\M\sim \mathcal{D}}[\M]\Vert_{F}^{2}+\max_{i\in[n]}\left(\left<\A_{i},\X\right>-b_{i}\right)},
\end{equation}
where $\mathcal{D}$ is again an unknown distribution over instances. Here the matrices $\{\A_i\}_{i\in[n]}$ and scalars $\{b_i\}_{i\in[n]}$ can absorb a penalty factor $\lambda$.

Here, by \cite{smoothing}, the $\mu$-smooth approximation of $\max_{i\in[n]}\left(\left<\A_{i},\X\right>-b_{i}\right)$ is 
\[R_{\mu}(\X)=\mu\log\left(\sum_{i=1}^{n}e^{\frac{1}{\mu}\left(\left<\A_{i},\X\right>-b_{i}\right)}\right) ,\] 
with parameters $(\Vert\mathcal{A}\Vert^{2},\log{n},0)$, where 
$\mathcal{A}:\mathbb{R}^{m\times d}\rightarrow\mathbb{R}^{n}$ is a linear transformation with the form
$\mathcal{A}(\X)=\left(\begin{array}{c}
tr(\A_{1}^{T}\X),\ 
tr(\A_{2}^{T}\X),\ 
\ldots,\ 
tr(\A_{n}^{T}\X) 
\end{array}\right)^{\top}$, for $\A_{1},...,\A_{n}\in\mathbb{R}^{m\times d}$, and
$\Vert\mathcal{A}\Vert=\max\{\Vert\mathcal{A}(\X)\Vert_{2}:\ \Vert\X\Vert_{F}=1\}$.
This satisfies
\begin{equation} \label{eq:linearBound}
\begin{split}
R_{\mu}(\X) \le \max_{i\in[n]}\left(\left<\A_{i},\X\right>-b_{i}\right) \le R_{\mu}(\X)+\mu\log{n}.
\end{split}
\end{equation}

In this case, $G(\X)$ and $h(\X)$ are as in Example $1$.

\begin{corollary} \label{cor:genEx2}
Consider running Algorithm \ref{alg:stoch} for the smooth approximation of Problem \eqref{eq:linearMax}, with parameters $T=\frac{8\ln{8}}{3}\left(\frac{\Vert\mathcal{A}\Vert^2}{\mu}+1\right)+1$, $\eta_t=\frac{\mu}{2\Vert\mathcal{A}\Vert^2+2\mu}$, $k_s=\frac{32\mathbb{E}[\Vert M-\mathbb{E}[M]\Vert^{2}]}{C_{0}}2^{s-1}$, $k_t=32$, and let $\X^*$ denote the optimal solution to Problem \eqref{eq:linearMax}. Then, running $S \ge \log_2{\left(\frac{C_0}{\varepsilon}\right)}+2$ epochs of the outer-loop guarantees that:
\begin{equation} \label{eq:AMEsmoothing}
\begin{split}
\mathbb{E}[f(\X_{S})]-f(\X^*) & \le \frac{\varepsilon}{4}+\frac{8}{7}\delta_{1}+\frac{23\log{n}}{7}\mu.
\end{split}
\end{equation}
In particular, taking a smoothing parameter of $\mu=\frac{7\varepsilon}{92\log{n}}$ and $\delta_1=\frac{7\varepsilon}{32}$, we obtain
\begin{equation} \label{eq:AMEsmoothing_particular}
\begin{split}
\mathbb{E}[f(\X_{S})]-f(\X^*) & \le \varepsilon.
\end{split}
\end{equation}
\end{corollary}

\begin{proof}
The parameters of the problem are as follows:
$\alpha=1$, $\beta_G=1$, $\gamma=\log{n}$, $\beta_R=\frac{\Vert\mathcal{A}\Vert^2}{\mu}$, $\sigma^2=\mathbb{E}[\Vert M-\mathbb{E}[M]\Vert^{2}]$.
Therefore, by \cref{thm:nonsmooth} we get the result in (\ref{eq:AMEsmoothing}). By choosing $\mu=\frac{7\varepsilon}{92\log{n}}$ and $\delta_1=\frac{7\varepsilon}{32}$, the result in (\ref{eq:AMEsmoothing_particular}) is immediate.
\end{proof}

\subsubsection{Example 3: Recovering a low-rank and sparse matrix from linear measurements with elastic-net regularization}

Finally, we very briefly discuss a matrix-sensing problem, where both a nuclear-norm constraint is used to promote low-rank solutions and the well known elastic-net regularizer \cite{zou2005regularization} is used to promote sparsity.
\begin{equation} \label{eq:elasticNet}
\min_{\Vert \X\Vert_{*}\le\tau}\mathbb{E}_{(\A,b)\sim \mathcal{D}}\left[{\frac{1}{2}\left({\langle{\A,\X}\rangle-b}\right)^2}\right]+\lambda_1\Vert{\X}\Vert_1 + \lambda_2\Vert{\X}\Vert_F^2.
\end{equation}
In this example, $G(\X) = \mathbb{E}_{(\A,b)\sim \mathcal{D}}\left[{\frac{1}{2}\left({\langle{\A,\X}\rangle-b}\right)^2}\right]$ need not be strongly convex as in previous examples, however the elastic-net regularizer $R(\X) := \lambda_1\Vert{\X}\Vert_1 + \lambda_2\Vert{\X}\Vert_F^2$ is strongly convex.

The smoothing of Problem \eqref{eq:elasticNet} and resulting application of our method goes along the same lines as our treatment of Problem \eqref{eq:matrixEstimation}.

\section{Experiments}\label{sec:experiments}

In support of our theory, in this section we present preliminary empirical experiments on the problem of \textit{low-rank and sparse matrix estimation}, Problem \eqref{eq:matrixEstimation}. We compare our Algorithm \ref{alg:stoch} (SVRGCG) to previous conditional gradient-type stochastic methods including the Stochastic Conditional Gradient Algorithm (SCG) \cite{Hazan_finiteSum}\footnote{In \cite{Hazan_finiteSum} it appears as Stochastic Frank-Wolfe (SFW).} and the Stochastic Conditional Gradient Sliding Algorithm (SCGS) \cite{Lan}.

We use synthetic randomly-generated data for the experiments. For all experiments the input matrix is of the form $\M_0 = \E_{\M\sim\mD}[\M] = \Y\Y^{\top} + \N$, where $\Y\in\reals^{d\times r}$ is a random sparse matrix for which each entry is zero w.p. $1-1/\sqrt{d}$ and $U\{1,\dots,10\}$ w.p. $1/\sqrt{d}$, $\N$ is a $d\times d$ random matrix with i.i.d. standard Gaussian entries. We set the dimension to $d=300$ and the rank of $\Y$, $r$ to either $1$ or $10$.
In all experiments we set $\lambda=2$, $\varepsilon=0.01\cdot\Vert{\Y\Y^{\top}}\Vert_F^2$ (i.e., the approximation error is relative to magnitude of signal), $\mu=\varepsilon/d^2$ (in accordance with Corollary \ref{cor:genEx1}) ,  and $\tau=\trace(\Y\Y^{\top})$. The stochastic oracle is implemented by taking noisy observations of $\M_0$ using: $\M^{(i)} = \M_0 + \sigma\Q^{(i)}$, where each $\Q^{(i)}$ is random with i.i.d. standard Gaussian entries and we fix $\sigma = 5$.

For all three methods we measure i) the obtained (original non-smooth) function value (see \eqref{eq:matrixEstimation}) vs. number of stochastic gradients used, ii) function value vs. overall runtime (seconds), and iii) function value vs. overall number of rank-one SVD computations used. Since the overall running time is highly dependent on specific implementation, we bring the number of rank-one SVD computations as an implementation-independent proxy for the overall runtime. For our method SVRGCG, we compute the overall number of rank-one SVD computations by multiplying the number of SVD factorizations with the rank of the factorization used\footnote{This is reasonable since the runtime for low-rank SVD typically scales linearly with rank.}. In the first experiment (Figure \ref{fig:1}) we set $\rank(\Y)=1$ in which case, all three algorithms use only rank-one SVD computations. In a second experiment (Figure \ref{fig:2}), we set $\rank(\Y)=10$ in which case, algorithms SCG, SCGS still use only rank-one SVD computations, whereas our algorithm SVRGCG uses rank-$10$ SVD computations (hence the left panel in Figure \ref{fig:2} counts 10 times the number of thin-SVD computations used by our algorithm). The results for each experiment are averages of 30 i.i.d. runs.

Importantly, all three algorithms were implemented as suggested by theory without attempts to optimize their performance, with only two exceptions. First, in our algorithm SVRGCG we use the rank of $\Y$ to set the rank of SVD computations (since naturally $\Y\Y^{\top}$ should be close to the optimal solution). Second,  in \cite{Lan} it is suggested to run the conditional gradient method in order to solve proximal sub-problems in algorithm SCGS until a certain quantity, which serves as a certificate for the quality of the solution is reached. However, we observe that in practice, obtaining this certificate takes unreasonable amount of iterations which renders the overall method highly suboptimal w.r.t. the alternatives. Hence in our implementation we limit the number of CG inner iterations to the dimension $d$.

The results are presented in Figures \ref{fig:1} and \ref{fig:2}. It can be seen that our algorithm SVRGCG clearly outperforms both SCG and SCGS with respect to all three measures in the two experiments.


\begin{figure}[H]
  \begin{subfigure}[t]{0.32\textwidth}
    \centering
    \includegraphics[width=\textwidth]{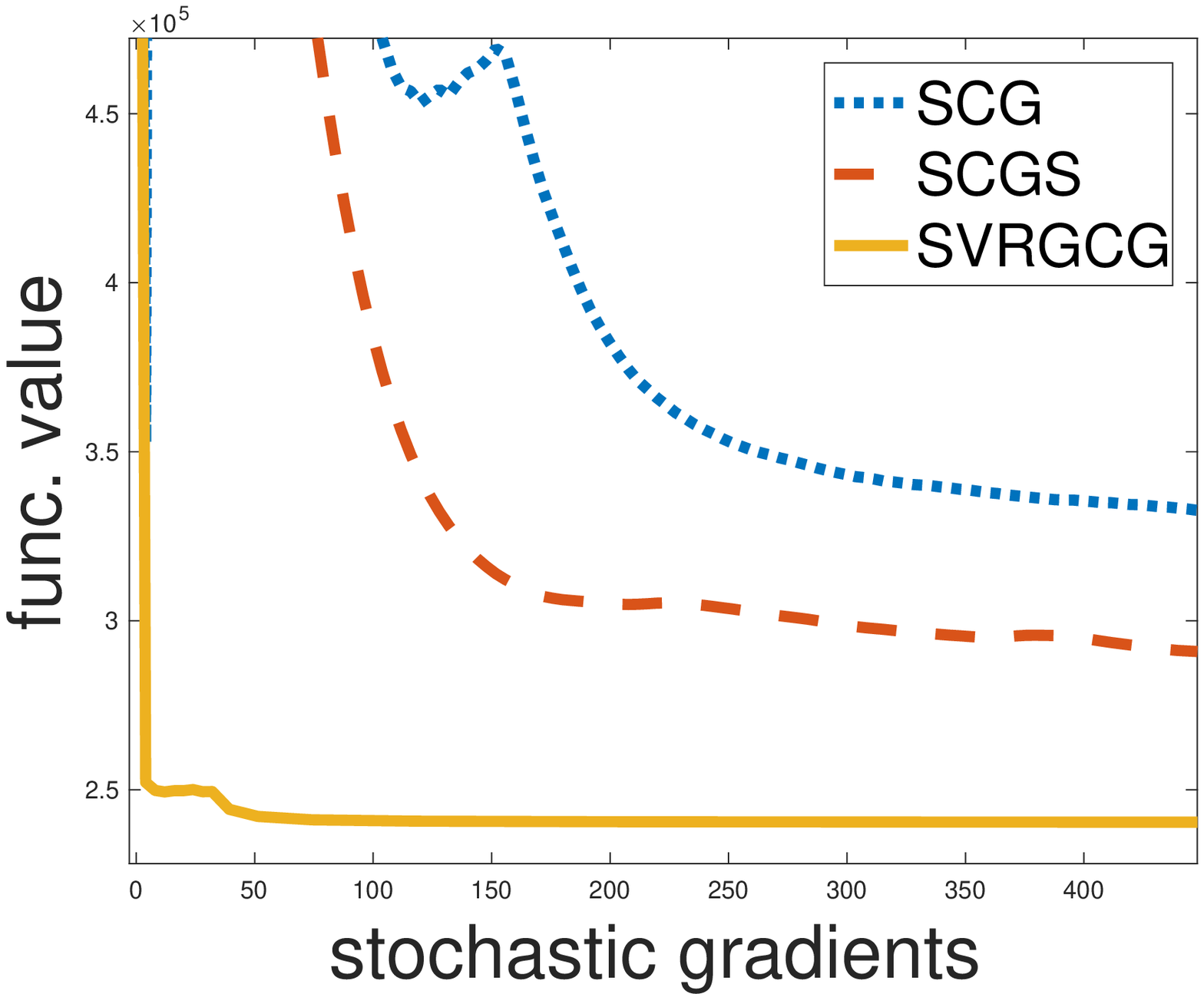}
  \end{subfigure}
  ~
  \begin{subfigure}[t]{0.32\textwidth}
    \centering
    \includegraphics[width=\textwidth]{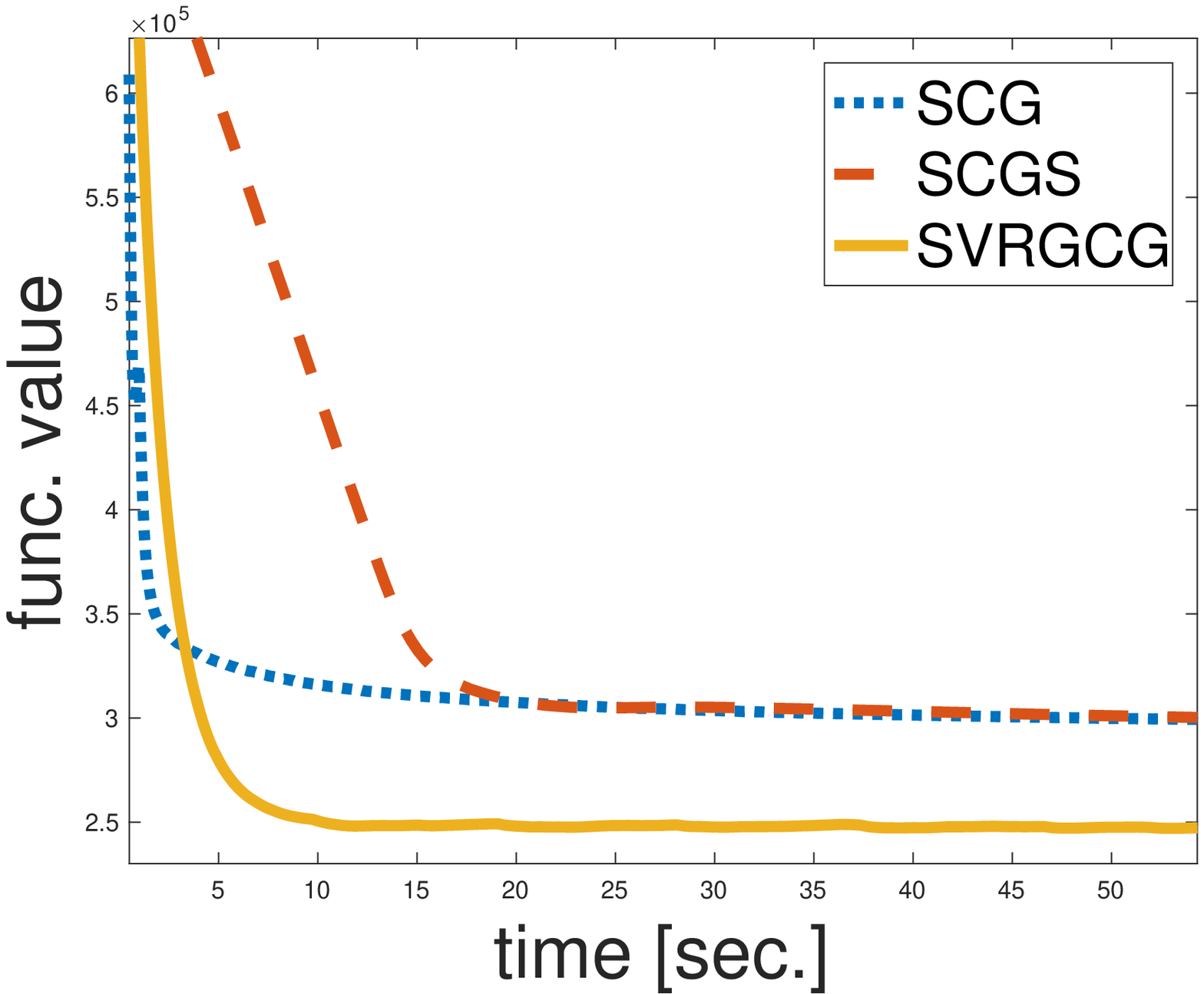}
  \end{subfigure}
  ~
  \begin{subfigure}[t]{0.32\textwidth}
    \centering
    \includegraphics[width=\textwidth]{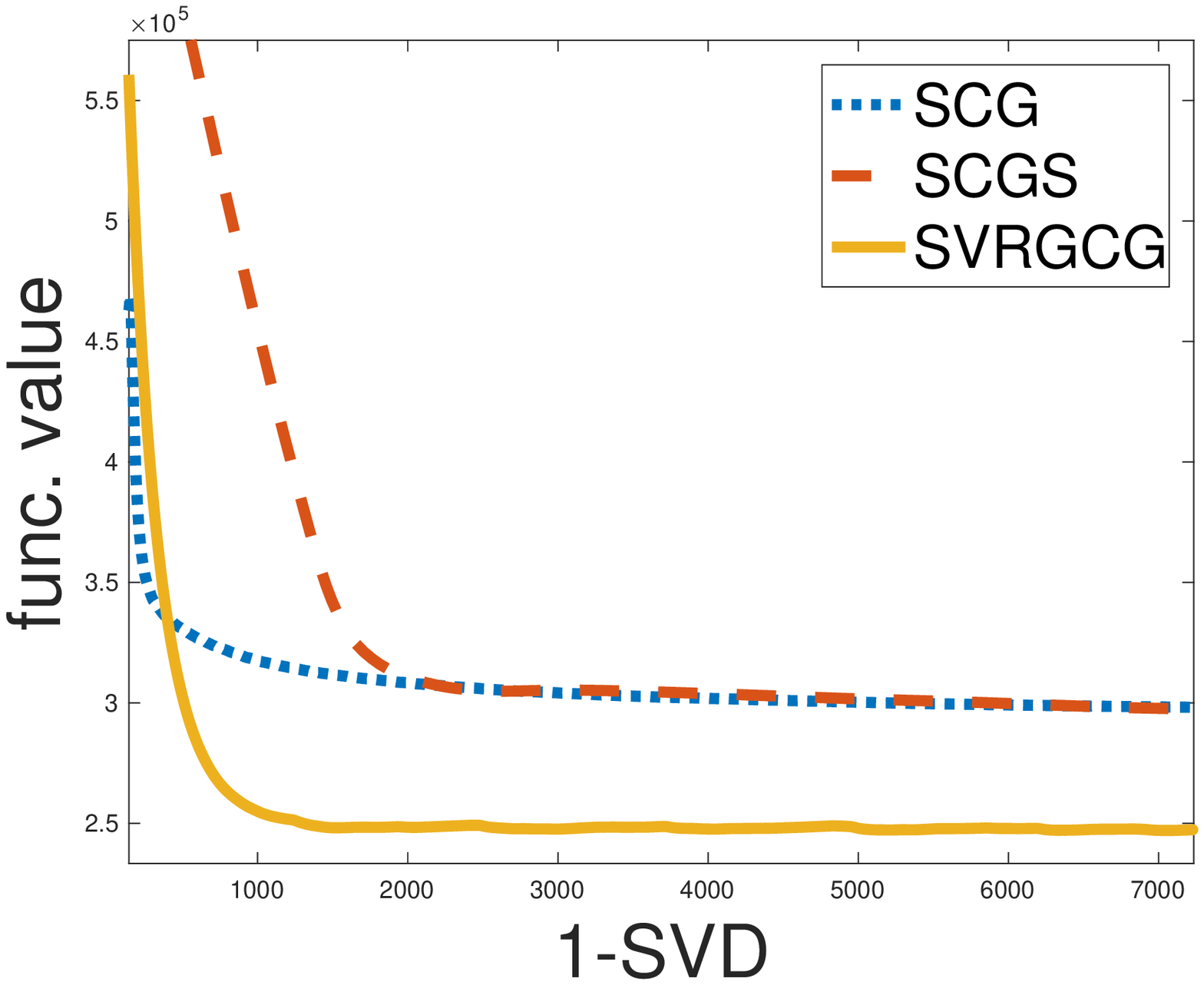}    
    \label{fig:1}
  \end{subfigure}
  \caption{Comparison between methods with $\rank(\Y\Y^{\top})=1.$}
  \label{fig:1}
\end{figure}

\begin{figure}[H]
  \begin{subfigure}[t]{0.32\textwidth}
    \centering
    \includegraphics[width=\textwidth]{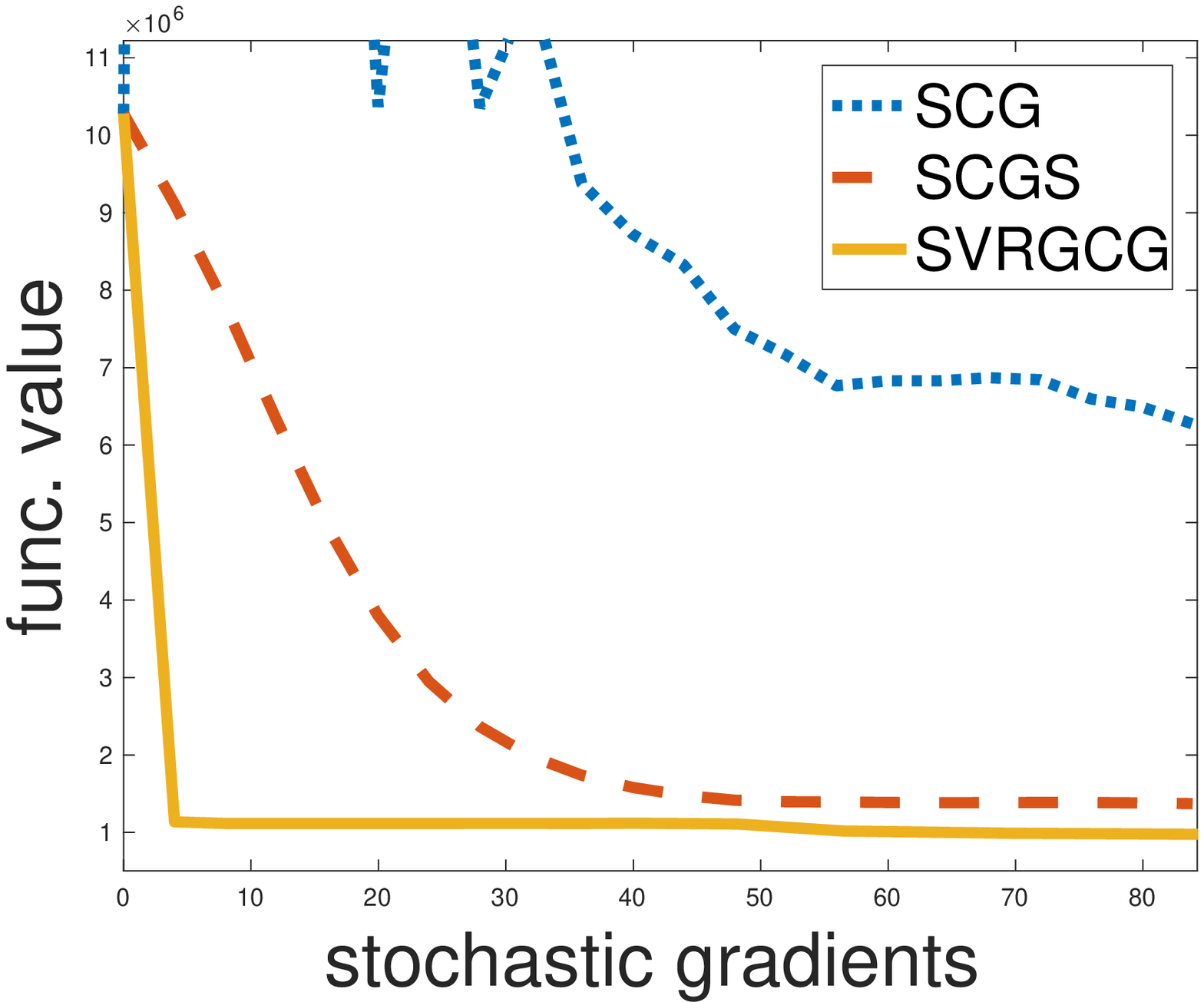}
  \end{subfigure}
  ~
  \begin{subfigure}[t]{0.32\textwidth}
    \centering
    \includegraphics[width=\textwidth]{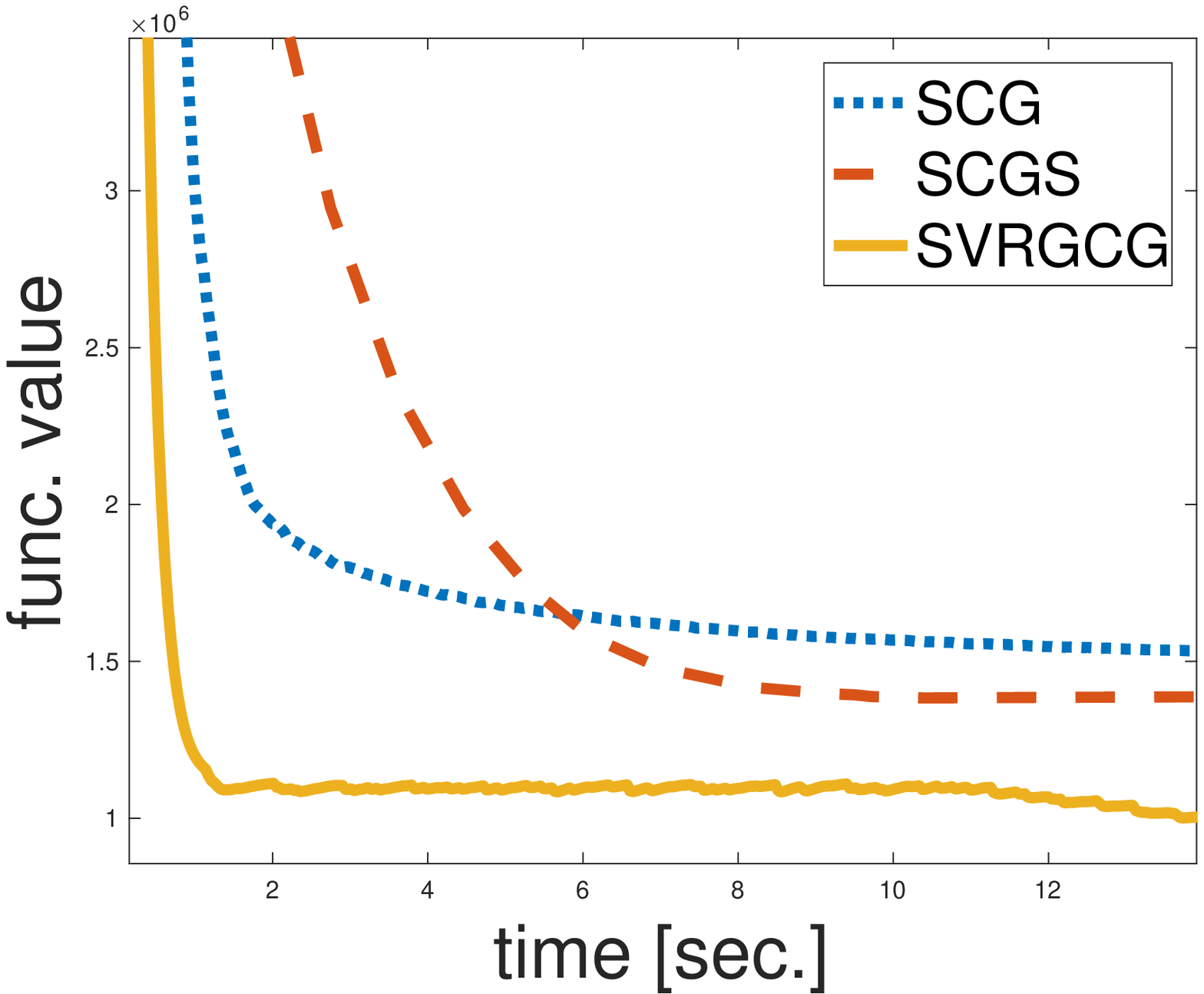}
  \end{subfigure}
  ~
  \begin{subfigure}[t]{0.32\textwidth}
    \centering
    \includegraphics[width=\textwidth]{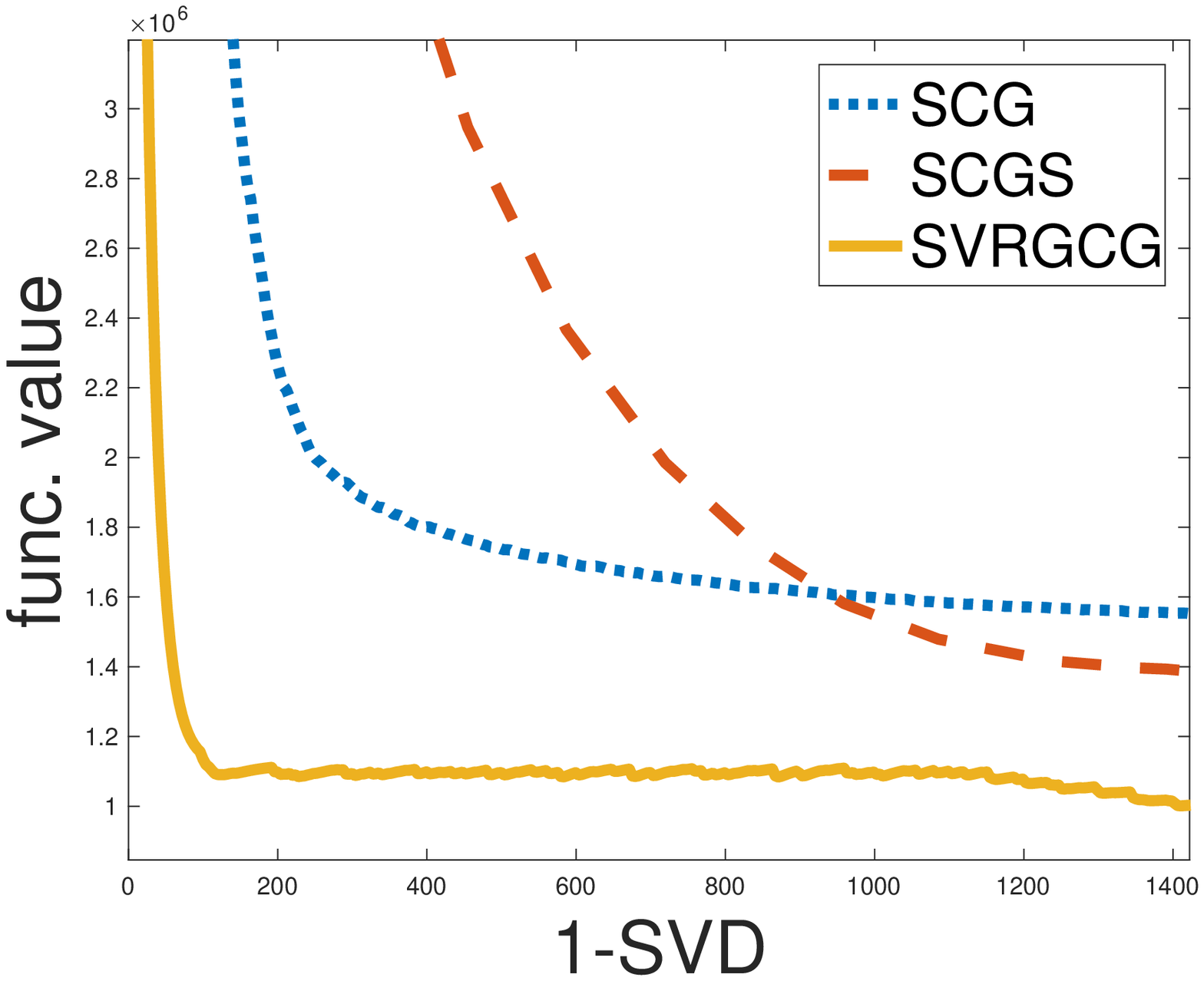}    
  \end{subfigure}
  \caption{Comparison between methods with $\rank(\Y\Y^{\top})=10.$}
  \label{fig:2}
\end{figure}

\paragraph{Acknowledgments:} We would like to thank Shoham Sabach for many fruitful discussions throughout the preparation of this manuscript.

\bibliographystyle{plain}
\bibliography{bibli}

\begin{thebibliography}{10}

\bibitem{FISTA}
Amir Beck and Marc Teboulle.
\newblock A fast iterative shrinkage-thresholding algorithm for linear inverse
  problems.
\newblock {\em SIAM journal on imaging sciences}, 2(1):183–202, 2009.

\bibitem{smoothing}
Amir Beck and Marc Teboulle.
\newblock Smoothing and first order methods: a unified framework.
\newblock {\em SIAM Journal on Optimization}, 22(2):557--580, 2012.

\bibitem{Bubeck15}
S{\'e}bastien Bubeck et~al.
\newblock Convex optimization: Algorithms and complexity.
\newblock {\em Foundations and Trends{\textregistered} in Machine Learning},
  8(3-4):231--357, 2015.

\bibitem{Frostig15}
Roy Frostig, Rong Ge, Sham~M Kakade, and Aaron Sidford.
\newblock Competing with the empirical risk minimizer in a single pass.
\newblock In {\em Conference on learning theory}, pages 728--763, 2015.

\bibitem{Garber16}
Dan Garber.
\newblock Faster projection-free convex optimization over the spectrahedron.
\newblock In {\em Advances in Neural Information Processing Systems 29: Annual
  Conference on Neural Information Processing Systems 2016, December 5-10,
  2016, Barcelona, Spain}, pages 874--882, 2016.

\bibitem{Hazan_finiteSum}
Elad Hazan and Haipeng Luo.
\newblock Variance-reduced and projection-free stochastic optimization.
\newblock {\em International Conference on Machine Learning}, pages 1263--1271,
  2016.

\bibitem{k-SVD}
Zeyuan Allen-Zhu{,} Elad Hazan{,}~Wei Hu and Yuanzhi Li.
\newblock Linear convergence of a frank-wolfe type algorithm over trace norm
  balls.
\newblock {\em NIPS}, pages 6192--6201, 2017.

\bibitem{Jaggi13}
Martin Jaggi.
\newblock Revisiting frank-wolfe: Projection-free sparse convex optimization.
\newblock {\em Proceedings of the 30th International Conference on Machine
  Learning, ICML}, pages 427--435, 2013.

\bibitem{Simon17}
Gauthier Gidel{,}~Tony Jebara and Simon Lacoste{-}Julien.
\newblock Frank-wolfe algorithms for saddle point problems.
\newblock {\em Proceedings of the 20th International Conference on Artificial
  Intelligence and Statistics, {AISTATS} 2017}, pages 362--371, 2017.

\bibitem{Simon18}
Gauthier Gidel{,}~Tony Jebara and Simon Lacoste{-}Julien.
\newblock Frank-wolfe splitting via augmented lagrangian method.
\newblock {\em International Conference on Artificial Intelligence and
  Statistics, {AISTATS} 2018}, pages 1456--1465, 2018.

\bibitem{SVRG}
Rie Johnson and Tong Zhang.
\newblock Accelerating stochastic gradient descent using predictive variance
  reduction.
\newblock In {\em Advances in neural information processing systems}, pages
  315--323, 2013.

\bibitem{Lan}
Guanghui Lan and Yi~Zhou.
\newblock Conditional gradient sliding for convex optomization.
\newblock {\em SIAM Journal on Optimization}, 26(2):1379--1409, 2016.

\bibitem{Cevher18}
Alp Yurtsever{,} Olivier Fercoq{,}~Francesco Locatello and Volkan Cevher.
\newblock A conditional gradient framework for composite convex minimization
  with applications to semidefinite programming.
\newblock {\em Proceedings of the 35th International Conference on Machine
  Learning, {ICML} 2018}, pages 5713--5722, 2018.

\bibitem{mu2016scalable}
Cun Mu, Yuqian Zhang, John Wright, and Donald Goldfarb.
\newblock Scalable robust matrix recovery: Frank--wolfe meets proximal methods.
\newblock {\em SIAM Journal on Scientific Computing}, 38(5):A3291--A3317, 2016.

\bibitem{Nesterov13}
Yurii Nesterov.
\newblock {\em Introductory lectures on convex optimization: A basic course},
  volume~87.
\newblock Springer Science \& Business Media, 2013.

\bibitem{Garber18}
Dan Garber{,}~Shoham Sabach and Atara Kaplan.
\newblock Fast generalized conditional gradient method with applications to
  matrix recovery problems.
\newblock {\em CoRR}, abs/1802.05581, 2018.

\bibitem{Richard12}
Emile Richard{,} Pierre-Andr`e Savalle and Nicolas Vayatis.
\newblock Estimation of simultaneously sparse and low rank matrices.
\newblock {\em Proceedings of the 29th International Conference on Machine
  Learning}, 2012.

\bibitem{Wang13}
Yu-Xiang Wang, Huan Xu, and Chenlei Leng.
\newblock Provable subspace clustering: When lrr meets ssc.
\newblock In {\em Advances in Neural Information Processing Systems}, pages
  64--72, 2013.

\bibitem{Zhou13}
Ke~Zhou, Hongyuan Zha, and Le~Song.
\newblock Learning social infectivity in sparse low-rank networks using
  multi-dimensional hawkes processes.
\newblock In {\em Artificial Intelligence and Statistics}, pages 641--649,
  2013.

\bibitem{zou2005regularization}
Hui Zou and Trevor Hastie.
\newblock Regularization and variable selection via the elastic net.
\newblock {\em Journal of the Royal Statistical Society: Series B (Statistical
  Methodology)}, 67(2):301--320, 2005.

\end{thebibliography}

\end{document}